\documentclass[10pt,fleqn]{article}
\usepackage[english]{babel}
\usepackage{amsmath}
\usepackage{amssymb}
\usepackage{amsthm}
\usepackage{caption}
\usepackage{subcaption}
\usepackage{tabularx,booktabs}
\usepackage{graphicx}
\usepackage{enumitem}
\usepackage{csquotes}
\usepackage{xspace}
\usepackage{hyperref}
\usepackage[nameinlink]{cleveref}
\usepackage{multirow}
\usepackage{placeins}
\usepackage{authblk}
\usepackage[a4paper,portrait,margin=1in]{geometry}
\usepackage{booktabs}

\newcommand{\R}{\mathbb{R}}

\renewcommand{\vec}[1]{\mathbf{#1}}
\newcommand{\mat}[1]{\mathbf{#1}}
\newcommand{\transpose}{^\top}

\newcommand{\abs}[1]{\left|#1\right|}
\newcommand{\norm}[2][]{\left\|#2\right\|_{#1}}

\renewcommand{\d}[3][]{\frac{\dd^{#1} #2}{\dd #3^{#1}}}

\newcommand{\pd}[3][]{\frac{\partial^{#1} #2}{\partial #3^{#1}}}

\newcommand{\dd}{\mathrm{d}}

\newcommand{\nrd}[1][]{\(#1^\text{rd}\)}
\newcommand{\nth}[1][]{\(#1^\text{th}\)}

\newcommand{\third}{\nrd[3]}

\newcommand{\bydef}{:=}
\newcommand{\etal}{et~al.~}

\newtheorem  {theorem}      {Theorem}

\numberwithin{theorem}      {section}
\numberwithin{lemma}        {section}
\numberwithin{corollary}    {section}
\numberwithin{definition}   {section}

\newcommand{\fwddiff}{\Delta_{\text{fwd}}}

\newcommand{\uref}{\vec{u}_{\textnormal{ref}}}
\newcommand{\tlyap}{T_{\textnormal{Lyap}}}

\newcommand{\optdisc}{optimise-then-discretise\xspace}
\newcommand{\discopt}{discretise-then-optimise\xspace}
\newcommand{\Optdisc}{Optimise-then-discretise\xspace}
\newcommand{\Discopt}{Discretise-then-optimise\xspace}
\newcommand{\ks}{Kuramoto-Sivashinsky\xspace}

\begin{document}

\title{Comparison of neural closure models for discretised PDEs}
\author[3]{Hugo Melchers\footnote{Work performed while at Centrum Wiskunde \& Informatica}}
\author[1,2]{Daan Crommelin}
\author[3]{Barry Koren}
\author[3]{Vlado Menkovski}
\author[1,*]{Benjamin Sanderse}
\affil[1]{
    Centrum Wiskunde \& Informatica, Science Park 123, 1098 XG Amsterdam, The Netherlands
}
\affil[2]{
    Korteweg-de Vries Institute for Mathematics, University of Amsterdam, Science Park 105-107, 1098 XG, The Netherlands
}
\affil[3]{
    Eindhoven University of Technology, De Zaale, 5600 MB, Eindhoven, The Netherlands
}
\affil[*]{
    Corresponding author. E-mail address: \url{mailto:b.sanderse@cwi.nl}
}

\maketitle

\begin{abstract}
    Neural closure models have recently been proposed as a method for efficiently approximating small scales in
    multiscale systems with neural networks.
    The choice of loss function and associated training procedure has a large effect on the accuracy and stability of
    the resulting neural closure model.
    In this work, we systematically compare three distinct procedures: ``derivative fitting'', ``trajectory fitting'' with \discopt,
    and ``trajectory fitting'' with \optdisc.
    Derivative fitting is conceptually the simplest and computationally the most efficient approach and is found to perform reasonably well on one of the test problems (Kuramoto-Sivashinsky) but poorly on the other (Burgers).
    Trajectory fitting is computationally more expensive but is more robust and is therefore the preferred approach.
    Of the two trajectory fitting procedures, the \discopt approach produces more accurate models than the \optdisc approach.
    While the \optdisc approach can still produce accurate models, care must be taken in choosing the length of the
    trajectories used for training, in order to train the models on long-term behaviour while still producing reasonably
    accurate gradients during training.
    Two existing theorems are interpreted in a novel way that gives insight into the long-term accuracy of a neural
    closure model based on how accurate it is in the short term.
\end{abstract}

\section*{Abbreviations}
\begin{tabular}{ll}
    \textbf{AD  } & Automatic differentiation \\
    \textbf{CNN } & Convolutional neural network \\
    \textbf{FOM } & Full-order model \\
    \textbf{LES } & Large eddy simulation \\
    \textbf{MOR } & Model order reduction \\
    \textbf{MSE } & Mean-square error \\
    \textbf{NN  } & Neural network \\
    \textbf{ODE } & Ordinary differential equation \\
    \textbf{PDE } & Partial differential equation \\
    \textbf{POD } & Proper orthogonal decomposition \\
    \textbf{RANS} & Reynolds-averaged Navier-Stokes \\
    \textbf{RMSE} & Root-mean-square error \\
    \textbf{RNN } & Recurrent neural network \\
    \textbf{ROM } & Reduced-order modelling \\
    \textbf{VPT } & Valid prediction time
\end{tabular}

\section{Introduction}
\label{sec:intro}

A number of real-world phenomena, such as fluid flows, can be modelled numerically as a system of partial differential
equations (PDEs).
Such PDEs are typically solved by discretising them in space, yielding ordinary differential equations (ODEs) over a
large number of variables.
These full-order models (FOMs) are generally very accurate, but can be computationally expensive to solve.
A remedy against this high computational cost is to use `truncated' models. These do not directly resolve
all spatial and/or temporal scales of the true solution of the underlying PDE, thereby lowering the dimensionality of the model.
Approaches to create lower dimensional models include reduced-order modelling (ROM \cite{salmoiraghi2016advances}), as well as
large eddy simulation (LES \cite{sagaut2006large}) and Reynolds-averaged Navier-Stokes (RANS \cite{alfonsi2009reynolds})
for fluid flow problems, specifically.
In such a truncated model, one or more closure terms appear, which represent the effects that are not directly resolved
by the reduced-order model.
For a recent overview of closure modelling for reduced-order models, see Ahmed \etal\cite{ahmed2021closures}.

While closure terms can in some cases be derived from theory (for example, for LES), this is generally not the case.
When they cannot be derived from theory, a recent approach is to use a machine learning model to learn the
closure term from data.
In this approach a specific type of machine learning model is used, called a neural closure model \cite{gupta2021neural}.
The overall idea is to approximate a PDE or large ODE system by a smaller ODE system, and to train a neural network to correct for the approximation error in the resulting ODE system.
Neural closure models are a special form of neural ODEs \cite{chen2018neural}, which have been the subject
of extensive research in the past years, for example by Finlay \etal\cite{finlay2020train} and Massaroli
\etal\cite{massaroli2020dissecting}.

A number of different approaches for training neural ODEs and neural closure models are available.
An important distinction is between approaches that compare predicted and actual time-derivatives of the
ODE (``derivative fitting''), and approaches that compare predicted and actual solutions (``trajectory fitting'').
Trajectory fitting itself can be done in two ways, depending on whether the optimisation problem for the neural network
is formulated as continuous in time and then discretised using an ODE solver (\optdisc), or formulated as discrete in
time (\discopt).

In several recent studies, neural closure models have been applied to fluid flow problems.
The considered approaches include
derivative fitting
\cite{guan2022learning,park2021toward,macart2021embedded,beck2018deep}, \discopt \cite{list2022learned}, and \optdisc
\cite{sirignano2020dpm,macart2021embedded}.
Derivative fitting is also used on a comparable but distinct problem by San and Maulik~\cite{san2018neural}.
There, Burgers' equation is solved using model order reduction (MOR) by means of proper orthogonal decomposition
(POD), resulting in an approximate ODE for which the closure term is approximated by a neural network.

Training neural ODEs efficiently and accurately has been the subject of some previous research. However, in the context of neural closure models, most of this earlier
work either does not consider certain relevant aspects or is not directly applicable.
For example, Onken and Ruthotto~\cite{onken2020discretize} compare \discopt and \optdisc for pure neural ODEs (i.e.~ODEs
in which the right-hand side only consists of a neural network term). They omit a derivative fitting approach since such
an approach is not applicable in the contexts considered there.
Ma \etal\cite{ma2021comparison} compare a wide variety of training approaches for neural ODEs, however with an emphasis on the
computational efficiency of different training approaches rather than on the accuracy of the resulting model.
Roesch \etal\cite{roesch2021collocation} compare trajectory fitting and derivative fitting approaches, considering pure
neural ODEs on two very small ODE systems.
As a result, the papers mentioned above are not fully conclusive to make general recommendations regarding how to train
neural closure models.

The purpose of this paper is to perform a systematic comparison of different approaches for constructing neural closure
models.
Compared to other works, the experiments performed here are not aimed at showing the efficacy of neural closure models
for a particular problem type, but rather at making general recommendations regarding different approaches for neural
closure models.
To this end, neural closure models are trained on data from two different discretised PDEs, in a variety of ways.
One of these PDEs, the Kuramoto-Sivashinsky equation, is chaotic and discretised into a stiff ODE system. This gives rise to additional challenges when
training neural closure models.
The results of this paper confirm that \discopt approaches are generally preferable to \optdisc approaches.
Furthermore, derivative fitting is found to be unpredictable, producing excellent models on one test set, but very poor
models on the other.
We give theoretical support to our results by reinterpreting two fundamental theorems from the fields of dynamical
systems and time integration in terms of neural closure models.

This paper is organised as follows: \Cref{sec:gradient-forms} describes a number of different approaches that are
available for training neural closure models.
\Cref{sec:theory} gives a number of theoretical results that can be used to predict the short-term and long-term
accuracy of models based on how they are trained and what error they achieve during training.
\Cref{sec:experiments} performs a number of numerical experiments in which the same neural closure model is trained in
multiple ways on the same two test equations, and the accuracy of the resulting models is compared.
Finally, \Cref{sec:conclusion} provides conclusions and recommendations.
The code used to perform the numerical experiments from \Cref{sec:experiments} is available online at
\url{https://github.com/HugoMelchers/neural-closure-models}.

\section{Preliminaries: approaches for neural ODEs}
\label{sec:gradient-forms}
In this paper, neural networks will be used as closure models for discretised PDEs.
Here, a time-evolution of the form \(\pd{u}{t} = F(u)\) is discretised into an ODE system \(\d{\vec{u}}{t} = f(\vec{u}),
\vec{u} \in \R^{N_x}\), such that taking progressively finer discretisations (resulting in larger values of \(N_x\))
produces more accurate solutions.
However, instead of taking a very fine discretisation, a relatively coarse discretisation will be used and a neural
network (NN) closure term will be added to correct for the spatial discretisation error.
This neural network depends not only on the vector \(\vec{u}\), but also on a vector \(\vartheta\) of trainable parameters:
\begin{align}
    \label{eq:generic-closure-model}
    \d{\vec{u}}{t} &= f(\vec{u}) + \text{NN}(\vec{u}; \vartheta).
\end{align}
Some of the theory regarding neural closure models also applies to neural ODEs, in which the neural network is the only
term in the right-hand side:
\begin{align}
    \label{eq:generic-neural-ode}
    \d{\vec{u}}{t} &= \text{NN}(\vec{u}; \vartheta).
\end{align}
In both cases, the result is a system of ODEs over a vector \(\vec{u}(t)\), in which the right-hand side depends not
only on \(\vec{u}(t)\) but also on some trainable parameters \(\vartheta\):
\begin{align}
    \label{eq:more-generic-neural-ode}
    \d{\vec{u}}{t} &= g(\vec{u}; \vartheta).
\end{align}
Note that in these models, the ODE is assumed to be autonomous, i.e., the right-hand side is assumed to be independent
of \(t\).
However, the work presented in this paper can be extended to non-autonomous ODEs, by extending the neural network to
depend on \(t\) or on some time-dependent control signal as well as on \(\vec{u}(t)\), and by including this additional data in the training data set.
The general form \eqref{eq:more-generic-neural-ode} covers more model types than just the neural ODEs and neural closure
models of equations \eqref{eq:generic-closure-model}~and~\eqref{eq:generic-neural-ode}.
Specifically, the output of the neural network does not have to be one of the terms in the right-hand side function, but
can also be included in other ways.
For example, Beck \etal\cite{beck2018deep} train neural networks to predict the eddy viscosity in an LES closure term,
rather than to predict the entire closure term, in order to ensure stability of the resulting model.

In this work, the specific form \eqref{eq:generic-closure-model} is used, with the exception that the output of the
neural network is passed through a simple linear function \(\fwddiff\), listed as a non-trainable layer in
\Cref{table:small-nn,table:large-nn}, which ensures that the solutions of the neural ODE satisfy conservation
of momentum.

Training a neural network corresponds to minimising a certain loss function, which must be chosen ahead of time.
Some loss functions are such that their gradients, which are used by the optimiser, can be computed in different ways.
In this section, an overview of different available approaches is given.

\begin{figure}
    \centering
    \includegraphics[width=0.48\textwidth]{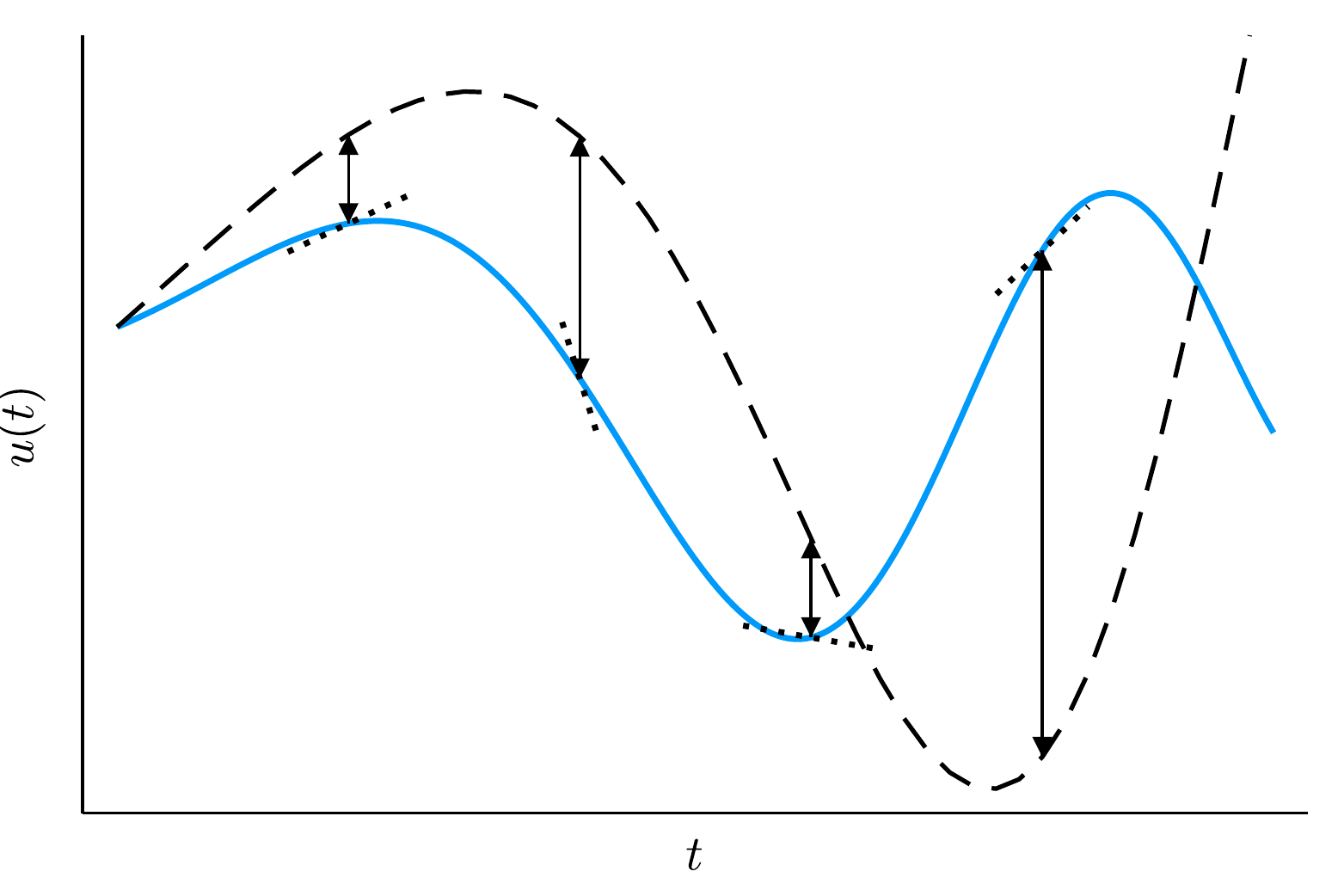}
    \caption{
        A visual comparison of the two types of neural ODE training: given a reference trajectory (solid), one can
        train the neural ODE to match the time-derivative of the trajectory (dotted lines), or to result in accurate ODE
        solutions (dashed line and arrows).
    }
    \label{fig:derivative-vs-trajectory-fitting}
\end{figure}

\subsection{Derivative fitting}
\label{sec:gradient-forms/derivfit}
With derivative fitting, also referred to as non-intrusive training \cite{pawar2019deep}, the loss function compares the
predicted and actual time-derivatives (i.e.~right-hand sides) of the neural ODE.
In this paper, the loss function used for derivative fitting will be a mean-square error (MSE):
\begin{align}
    \label{eq:derivative-fitting}
    \text{Loss}\left(\vartheta, \uref, \d{}{t}\uref\right)
    &= \frac{1}{N_xN_sN_p}\sum_{i = 1}^{N_s}\sum_{j = 1}^{N_p}\norm[2]{\d{\uref^{(j)}}{t}(t_i)
    - g\left(\uref^{(j)}(t_i); \vartheta\right)}^2.
\end{align}
Here, \(N_x\) is the size of the vector \(\uref\), \(N_s\) is the number of snapshots in each trajectory of the
training data, and \(N_p\) is the number of trajectories (i.e.~ODE solutions).
The value and time-derivative of the \nth[j] trajectory at time \(t_i\) are given by \(\uref^{(j)}(t_i)\) and
\(\d{\uref^{(j)}}{t}(t_i)\), respectively.

The main advantage of derivative fitting is that in order to compute the gradient of the loss function with respect to
\(\vartheta\), one only has to differentiate through the neural network itself.
This makes derivative fitting a relatively simple approach to use.
A disadvantage of derivative fitting is that the training data must consist of not just the values \(\vec{u}\), but also
their time derivatives \(\d{\vec{u}}{t}\).
This data is not always available, for example in cases where the trajectories \(\vec{u}(t)\) are obtained as
measurements from a physical experiment.
In this work, the training data is obtained through a high-resolution numerical algorithm.
Hence, the derivatives to be used for training are available.
In cases where exact derivatives are not available, they can be estimated from the available data for \(\vec{u}(t)\)
itself, as described by Roesch \etal\cite{roesch2021collocation}.
While they obtain good results with approximated derivatives, in general it is to be expected that substituting real
time-derivatives by approximations also decreases the accuracy of the neural network.

\subsection{Trajectory fitting: \Discopt}
\label{sec:gradient-forms/discopt}
An alternative to derivative fitting is trajectory fitting, also referred to as intrusive training \cite{pawar2019deep},
embedded training \cite{macart2021embedded}, or a solver-in-the-loop setup \cite{um2020solver}.
Here, the loss function compares the predicted and actual trajectories of the neural ODE.
Unless otherwise specified, trajectory fitting will also be done with the MSE loss function:
\begin{align}
    \label{eq:trajectory-fitting}
    \text{Loss}\left(\vartheta, \uref\right)
    &= \frac{1}{N_xN_tN_p}\sum_{i = 1}^{N_t}\sum_{j = 1}^{N_p}\norm[2]{\vec{u}^{(j)}(t_i) - \uref^{(j)}(t_i)}^2, \\
    \text{ where } \d{\vec{u}^{(j)}}{t}
    &= g\left(\vec{u}^{(j)}; \vartheta\right) \text{ and } \vec{u}^{(j)}(0) = \uref^{(j)}(0).
\end{align}
Trajectory fitting involves applying an ODE solver to the neural closure model to perform \(N_t\) time steps, where
\(N_t\) is a hyper-parameter that must be chosen ahead of time.
Computing the gradient of the loss function involves differentiating through the ODE solving process and can be done in
two separate ways.
One way to do this is by directly differentiating through the computations of the ODE solver.
This approach is called \textbf{\discopt}.

In the \discopt approach, the neural ODE is embedded in an ODE solver, for example an explicit Runge-Kutta method.
In such an ODE solver, the next solution snapshot \(\vec{u}(t + \Delta t)\) is computed from \(\vec{u}(t)\) by
performing one step of the ODE solver, which generally involves applying the internal neural network multiple times
(depending on the number of stages of the ODE solver).
This is repeated to obtain a predicted trajectory over a time interval of length \(T = N_t\Delta t\).
Since all the computations done by an ODE solver are differentiable, one can simply compute the required gradient by
differentiating through all the time steps performed by the ODE solver.
The \discopt approach effectively transforms a neural ODE into a discrete model, in which the time series is predicted
by advancing the solution by a fixed time step \(\Delta t\) at a time.
As such, any training approach that can be applied to discrete models of the form \(\vec{u}(t + \Delta t) = \text{model}
(\vec{u}(t))\) can also be applied to neural ODEs trained using this approach.

\subsection{Trajectory fitting: \Optdisc}
\label{sec:gradient-forms/optdisc}
Differentiating through the computations of the ODE solver is not always a possibility, for example if the ODE
solver is implemented as black-box software.
In such cases, trajectory fitting with the loss function \eqref{eq:trajectory-fitting} can still be used by computing gradients with the \textbf{\optdisc} approach.
In this approach, the required gradients are computed either by extending the ODE with more variables that store
derivative information, or by solving a second ``adjoint'' ODE backwards in time after the original ``forward'' ODE
solution is computed.
These two methods can be considered continuous-time analogues to forward-mode and reverse-mode automatic differentiation
(AD), respectively.

The adjoint ODE approach was popularised by Chen \etal\cite{chen2018neural}, who demonstrate that on some problems the
adjoint ODE approach can be used to train a neural ODE with much lower memory usage than other approaches.
Ma~\etal\cite{ma2021comparison} find that the adjoint ODE approach is computationally more efficient than the forward
mode approach for ODEs with more than 100 variables and parameters.
As such, a description of the forward mode approach is omitted here.
The adjoint ODE approach is the \optdisc approach that will be tested here.
This approach can be implemented in three different ways.
The implementation used in this work is the \textit{interpolating adjoint method}, in which the gradient of the loss
function is computed by first solving the forward ODE \eqref{eq:more-generic-neural-ode} to obtain the trajectory
\(\vec{u}(t)\), and then solving the adjoint ODE system
\begin{subequations}
    \label{eq:backsolve-ode}
    \begin{align}
        \d{}{t}\vec{y}\transpose &= -\vec{y}\transpose\pd{}{\vec{u}}g(\vec{u}; \vartheta), \quad \vec{y}(T) = \vec{0}, \label{eq:backsolve-ode-y} \\
        \d{}{t}\vec{z}\transpose &= -\vec{y}\transpose\pd{}{\vartheta}g(\vec{u}; \vartheta), \quad \vec{z}(T) = \vec{0}, \label{eq:backsolve-ode-z}
    \end{align}
\end{subequations}
from \(t = T\) backwards in time until \(t = 0\), performing discrete updates to \(\vec{y}(t)\) at times \(t_i, i = N_t,
N_t - 1, \dots, 2, 1\).
After the adjoint ODE system is solved, the gradient \(\d{\text{Loss}}{\vartheta}\) is given by \(\vec{z}(0)\).
For implementation details and an overview of other \optdisc methods, see Chapter 3 of Melchers~\cite{mastersthesis}.

Note that the two trajectory fitting approaches, i.e.~\discopt and \optdisc, both require choosing \(N_t\), the number
of time steps that the solution prediction is computed for.
As will be described in \Cref{sec:theory}, choosing \(N_t\) either too small or too large may have negative consequences
for the accuracy of the trained model.
For the \optdisc approach, the gradients used by the optimiser are computed as the solution of an ODE over a time span
of \(T = N_t \Delta t\).
Since the numerically computed ODE solution is inexact, choosing a larger value of \(N_t\) will generally result in less
accurate gradients, which may also decrease the accuracy of the trained model.

\subsection{Algorithm comparison}
\label{sec:gradient-forms/comparison}
\begin{table}
    \caption{
        An overview of the differences between the three training approaches outlined in \Cref{sec:gradient-forms}.
    }
    \label{table:gradient-form-comparison}
    \begin{tabular}{@{}p{6.5cm} c c c@{}}
        \toprule
        & & \multicolumn{2}{c}{Trajectory fitting}\\
        \cmidrule{3-4}
        & Derivative fitting & \multirow{2}{3cm}{\Discopt} & \multirow{2}{3cm}{\Optdisc} \\
        &  & &  \\
        \midrule
        Differentiability required & \(\text{NN}\) & \(\text{NN}\), \(f\), ODE solver & \(\text{NN}\), \(f\) \\
        Accuracy of loss function gradients & Exact & Exact & Approximate \\
        Learns long-term accuracy & No & Yes & Yes \\
        Requires time-derivatives of training data & Yes & No & No \\
        Computational cost & Low & High & High\\
        \bottomrule
    \end{tabular}
\end{table}
An overview of the advantages and disadvantages of different approaches is given in
\Cref{table:gradient-form-comparison}.
Here, the term `long-term' refers to the accuracy of predictions when solving the ODE over multiple time steps
as opposed to only considering the instantaneous error in the time-derivative of the solution.
Note that the computational cost will not be compared in this work; the goal is to compare the accuracy of the
resulting models.
Performance measurements of different training procedures will not be given here since the code used to perform the
numerical experiments in this work was not written with computational efficiency in mind, and since training was not
performed on recent hardware.
However, derivative fitting is expected to be computationally more efficient due to the fact that it does not require
differentiating through the ODE solution process.
A performance comparison between different implementations for \optdisc and \discopt approaches is given by
Ma~\etal\cite{ma2021comparison}.
The performance difference between derivative fitting and trajectory fitting will be taken into account when making
recommendations in \Cref{sec:conclusion}.

As for accuracy, the \discopt approach is expected to yield more accurate gradients than \optdisc, due to the absence of
temporal discretisation errors in the gradient computation.
The accuracy of derivative fitting is not easily compared to that of the other methods.
Like \optdisc, it suffers from the fact that the training does not take the temporal discretisation error of the ODE
solver into account.

Onken and Ruthotto~\cite{onken2020discretize} compare \discopt and \optdisc approaches for two problems, including a time
series regression similar to the trajectory fitting problem described earlier in this section.
Their findings indicate that training with \discopt results in computationally more efficient training (less time
required per epoch), as well as faster convergence (fewer epochs required to reach a given level of accuracy).
However, the trajectory regression test performed there only considers a small and relatively simple ODE system over
just two variables.
Furthermore, the use of neural networks as closure terms introduces some additional challenges that need to be overcome
in some cases, including solving the ODE \eqref{eq:generic-neural-ode} efficiently for stiff ODEs, and choosing the
value of \(N_t\) for chaotic systems.

\section{Theory concerning training approaches}
\label{sec:theory}

As described in the previous section, different training approaches are available for neural ODEs and neural closure
models.
The training approach used will generally have an effect on both the short-term accuracy and the long-term accuracy of
the trained model.
This is supported by the following two theorems, which use the short-term error of a model to provide upper bounds on
the long-term error.
Here, `short-term' refers to the predictions and prediction errors in the time-derivatives (for derivative fitting), or
after a single time step (for trajectory fitting).
The term `long-term' refers to the predictions after multiple time steps, i.e.~when predicting over a time interval of
\(T > \Delta t\).
Although neither of these theorems are new, they are interpreted in a novel way that gives insight regarding the
long-term accuracy of models based on their accuracy during training.

\subsection{Derivative fitting}
\label{sec:theory/derivfit}
For models trained using derivative fitting, a relation between the error of the right-hand sides of the ODE and the
error of the ODE solutions is given by Theorem~10.2 of Hairer \etal\cite{hairer1993solving}.
This theorem is referred to there as the `fundamental lemma' and is repeated here in a simplified form:
\begin{theorem}[Hairer \etal\cite{hairer1993solving}]
    \label{thm:error-bound-continuous}
    Let \(\uref(t) \in \R^{N_x}, t \geq 0\) be given, let \(\vec{u}(t) \in \R^{N_x}, t \geq 0\) be the solution of
    the ODE \(\d{\vec{u}}{t} = g(\vec{u}; \vartheta)\) with initial condition \(\vec{u}(0) = \uref(0)\), and let
    \(\norm{\cdot}\) be a norm on \(\R^{N_x}\).
    If the following holds:
    \begin{align*}
        a)\quad & \norm{\d{}{t}\uref(t) - g(\uref(t); \vartheta)} \leq \varepsilon, \\
        b)\quad & \norm{g(\vec{a}; \vartheta) - g(\vec{b}; \vartheta)} \leq C \norm{\vec{a} - \vec{b}},
    \end{align*}
    for fixed Lipschitz constant \(C > 0\) and fixed \(\varepsilon > 0\).
    Then the following error bound holds:
    \begin{align*}
        \norm{\uref(t) - \vec{u}(t)}
        &\leq \frac{\varepsilon}{C}\left( e^{Ct} - 1 \right).
    \end{align*}
\end{theorem}
\Cref{thm:error-bound-continuous} can be interpreted as follows: suppose that a machine learning model is trained to
predict the time-derivatives of trajectories \(\uref(t)\).
The result is a function \(g(\cdot; \vartheta)\) that computes the right-hand side of an ODE.
Then, achieving an error \(\leq \varepsilon\) in the short term does not guarantee a low error in the long term.
The error at time \(t\) \textit{can} grow exponentially in \(t\), with the exponential growth factor being dependent on
the Lipschitz constant of the trained model.

A possible mitigation for this problem is to add a term to the loss function that penalises large Lipschitz constants in
the neural network.
However, in the case of neural closure models~\eqref{eq:generic-closure-model}, penalising the Lipschitz constant of the
neural network only is not expected to help much.
\Cref{thm:error-bound-continuous} concerns the Lipschitz constant of the total right-hand side \(g(\cdot; \vartheta)\),
which equals \(f(\vec{u}) + \text{NN}(\vec{u}; \vartheta)\) for neural closure models.
In such cases, the Lipschitz constant of \(g(\cdot; \vartheta)\) cannot be kept small by bounding the Lipschitz constant
of the neural network term only.
Furthermore, for closure models for stiff ODEs the function \(f\) itself has a large Lipschitz constant by definition,
meaning that the Lipschitz constant of the total right-hand side will inevitably be large as well.
For many problems the function \(f\) is not Lipschitz continuous at all, for example because it contains a quadratic
term.
In such cases, the above theorem does not provide an upper bound for the error.
Additionally, even if reducing the Lipschitz constant of the entire model is an option, this may come at the
expense of lower short-term accuracy (i.e.~a larger value for \(\varepsilon\)).

An example where the potential exponential error increase seems to occur is encountered by Beck
\etal\cite{beck2018deep}: they train neural networks to predict the closure term in three-dimensional fluid flows,
and find that neural networks that predict this closure term with high accuracy can nonetheless result in inaccurate
predicted trajectories.
Similarly, Park and Choi~\cite{park2021toward} find that neural closure models with high accuracy on derivatives do not
always produce accurate trajectories.
MacArt \etal\cite{macart2021embedded} encounter the same issue and instead train using \optdisc to obtain models that
produce accurate solutions.
It should be noted, however, that derivative fitting does not always lead to poor models.
For example, Guan \etal\cite{guan2022learning} do not encounter this problem when training neural networks for LES
closure terms.

\subsection{Trajectory fitting}
\label{sec:theory/trajfit}
A theorem similar to \Cref{thm:error-bound-continuous} exists for models trained using trajectory fitting, either by
\discopt or by \optdisc.
This theorem is likely not new, but an independently derived proof is given here.
The theorem applies to all models that are used to advance a solution \(\vec{u}(t)\) by a fixed time step \(\Delta t\).
As such, the theorem is not written specifically for neural closure models.
\begin{theorem}
    \label{thm:error-bound-discrete}
    Let \(\uref(t_i), i = 0, 1, \dots, \dots\) be a sequence of vectors in \(\R^{N_x}\) where \(t_i = i\Delta t\), let
    \(\norm{\cdot}\) be a norm on \(\R^{N_x}\), and let \(G(\cdot; \vartheta): \R^{N_x} \to \R^{N_x}\) be a function
    such that:
    \begin{align*}
        a)\quad & \norm{\uref(t_{i+1}) - G(\uref(t_i); \vartheta)}
            \leq \varepsilon \text{ for all } i = 1, 2, \dots, N_t, \\
        b)\quad & \norm{G(\vec{a}; \vartheta) - G(\vec{b}; \vartheta)}
            \leq C \norm{\vec{a} - \vec{b}} \text{ for all } \vec{a}, \vec{b} \in \R^{N_x},
    \end{align*}
    for fixed Lipschitz constant \(C > 0, C \neq 1\) and fixed \(\varepsilon > 0\).
    Define the sequence \(\vec{u}(t_{i+1}) = G(\vec{u}(t_i); \vartheta)\) with \(\vec{u}(0) = \uref(0)\).
    Then the following error bound holds:
    \begin{align*}
        \norm{\vec{u}(t_k) - \uref(t_k)} &\leq \varepsilon \frac{C^k - 1}{C - 1} \text{ for } k = 0, 1, 2, \dots.
    \end{align*}
\end{theorem}
\begin{proof}
    For arbitrary \(k \geq 1\), the following holds:
    \begin{align*}
        \norm{\vec{u}(t_k) - \uref(t_k)}
        &= \norm{G(\vec{u}(t_{k-1}); \vartheta) - \uref(t_k)} \\
        &\leq \norm{G(\vec{u}(t_{k-1}); \vartheta) - G(\uref(t_{k-1}); \vartheta)}
            + \norm{G(\uref(t_{k-1}); \vartheta) - \uref(t_k)} \\
        &\leq C \norm{\vec{u}(t_{k-1}) - \uref(t_{k-1})} + \varepsilon.
    \end{align*}
    Now, define \(r_k = \norm{\vec{u}(t_k) - \uref(t_k)} + \frac{\varepsilon}{C - 1}\).
    Then \(r_0 = \frac{\varepsilon}{C - 1}\) and for \(k \geq 1\):
    \begin{align*}
        r_{k}
        &= \norm{\vec{u}(t_k) - \uref(t_k)} + \frac{\varepsilon}{C - 1} \\
        &\leq C\norm{\vec{u}(t_{k - 1}) - \uref(t_{k - 1})} + \varepsilon + \frac{\varepsilon}{C - 1} \\
        &= C \left( r_{k-1} - \frac{\varepsilon}{C - 1} \right) + \frac{C\varepsilon}{C - 1} \\
        &= C r_{k-1}.
    \end{align*}
    As a result, \(r_k \leq C^k r_0 = \frac{C^k \varepsilon}{C - 1}\), so:
    \begin{align*}
        \norm{\vec{u}(t_k) - \uref(t_k)}
        &= r_k - \frac{\varepsilon}{C - 1}
        \leq \varepsilon\frac{C^k - 1}{C - 1}.
    \end{align*}
\end{proof}
The proof does not work in the case that \(C = 1\) due to the use of the term \(C - 1\) as a numerator, but using
similar reasoning it can be shown that \(\norm{\vec{u}(t_k) - \uref(t_k)} \leq \varepsilon k\) holds when \(C = 1\).
Also note that if \(C < 1\) then \(G(\cdot; \vartheta)\) is a contraction and the sequence \(\vec{u}(t_k)\) will
converge to a fixed point.
This also results in a bounded error \(r_k\), since properties \(a)\) and \(b)\) combined imply that the sequence
\(\uref(t_k)\) is bounded as well.
However, in this case the model may still be qualitatively poor, since the actual and predicted sequences may converge
to different fixed points and may have different transient dynamics.
As such, even in the case that \(C < 1\), a low error after one time step does not imply that the behaviour over
multiple time steps is accurate.

\Cref{thm:error-bound-discrete} concerns the case that a model is trained by predicting a single time step, i.e.~with
\(N_t = 1\).
However, it can be generalised to models trained with \(N_t > 1\).
For example, a model trained with \(N_t = 2\) that achieves an error \(\leq \varepsilon_2\) after two time steps can be
seen as a single model that performs two applications of the inner model, meaning that \Cref{thm:error-bound-discrete}
can be applied to the model \(\vec{u} \mapsto G(G(\vec{u}; \vartheta); \vartheta)\), which has some Lipschitz constant
\(C_2\).
Then, the error at time step \(k\) is bounded by \(\varepsilon_2 (C_2^{k/2} - 1) / (C_2 - 1)\), since the approximation
at \(t = t_k\) only requires \(k/2\) applications of the model.

\Cref{thm:error-bound-discrete} implies that when training models by trajectory fitting, training with a small number of predicted time
steps \(N_t\) may result in models that produce poor predictions in the long term.
While training with larger values of \(N_t\) still results in an exponential upper bound for the error, it is expected that a model will yield more accurate solutions in the long term if long-term errors are penalised during training.

However, if the underlying ODE or PDE is chaotic, \(N_t\) must not be too large, either.
Initially similar solutions to a chaotic system diverge from each other as \(\exp(\lambda_{\max}t)\) where
\(\lambda_{\max}\) is the Lyapunov exponent of the system, which will be further explained in \Cref{sec:experiments/ks}.
Then, the sensitivity of the ODE solution after time \(t\) with respect to \(\vartheta\) will also grow as \(\exp(\lambda_{\max}t)\).
This means that when the loss function is a simple mean-square error between the predicted and actual trajectories, the
gradient of this loss function with respect to the model parameters is mostly determined by the solution of the neural
ODE for large \(t\).
The result is that the optimisation procedure works to decrease the long-term error instead of first decreasing the
short-term error.
For non-closure models (i.e.~pure neural ODEs of the form \(g(\vec{u}; \vartheta) = \text{NN}(\vec{u}; \vartheta)\)),
this may result in poor models (see Section~5.6.2 of Melchers~\cite{mastersthesis}).
This issue does not appear to be as severe for neural closure models.
Nonetheless it can be helpful to compensate for the exponentially increasing sensitivity by exponentially
weighing the loss function:
\begin{subequations}
    \begin{align}
        \label{eq:exp-weighted-mse}
        \text{Loss}_c\left(\vartheta, \uref\right)
        &= \frac{1}{N_xN_pZ}\sum_{i = 1}^{N_t}\sum_{j = 1}^{N_p}\exp(-2c\lambda_{\max}t_i)
            \cdot\norm[2]{\vec{u}^{(j)}(t_i) - \uref^{(j)}(t_i)}^2, \\
        \text{where } Z &= \sum_{i = 1}^{N_t} \exp(-2c\lambda_{\max}t_i).
    \end{align}
\end{subequations}
This loss function generalises the mean square error \eqref{eq:trajectory-fitting} by allowing the prediction error at
time \(t_i\) to be weighted by a factor \(\exp(-2c\lambda_{\max}t_i)\).
Here, the constant \(c\) can be chosen arbitrarily.
Taking \(c = 0\) recovers the standard mean-square error (MSE).
Note that the sum is over squared errors, which grow as \(\exp(2\lambda_{\max}t)\), meaning that the reasonable choice
according to the above reasoning is \(c = 1\).
Weighted loss functions with a number of choices for \(c\) will be evaluated in \Cref{sec:experiments/ks/optdisc}, as
well as training on shorter trajectories.

Similar to the continuous case (\Cref{thm:error-bound-continuous}), the exponential increase in error can be mitigated by
training the model in a way that penalises large Lipschitz constants.
Again, for neural closure models it is not sufficient to limit the Lipschitz constant of just the neural network.
Another mitigation approach is to train on a larger number of time steps, i.e.~to increase the value of \(N_t\) in the
loss function \eqref{eq:trajectory-fitting}.
Note that the fact that models trained to predict time series perform poorly when trained on single steps is already
known, and methods to mitigate this problem have already been studied.
In research such as that done by List \etal\cite{list2022learned}, unrolling multiple time steps is found to be crucial
in obtaining models that make accurate predictions.
Similarly, Pan and Duraisamy~\cite{pan2018long} find that in the discrete case, the long-term accuracy of the model
can be improved by adding a regularisation term to the loss function based on the Frobenius norm of the Jacobian of the
neural network.
Another way to obtain more accurate models is to add noise to the neural network's inputs (the vectors \(\vec{u}
_{\text{ref}}(\cdot)\)) during training, which makes the model less sensitive to small perturbations in its input
(see for example Chapter~7.4 of Goodfellow \etal\cite{Goodfellow-et-al-2016}).
Since the input to the model is equal to the output of the previous step, this regularisation technique is therefore
expected to decrease the rate at which the approximation error increases per step, meaning that it has a similar effect to
reducing the Lipschitz constant \(C\).

The exponentially increasing sensitivity with respect to the parameters, combined with \Cref{thm:error-bound-discrete}
shows that choosing the number of time steps \(N_t\) for trajectory fitting is not trivial as both too small and too
large values of \(N_t\) may result in poor models.

\section{Numerical experiments}
\label{sec:experiments}

In this section, neural networks will be trained in different ways to predict solutions of discretised PDEs of the form
\begin{align}
    \label{eq:generic-1D-pde}
    \pd{u}{t}(x, t) = F(u)(x, t).
\end{align}
This is a scalar PDE on a one-dimensional spatial domain.
The boundary conditions are periodic, i.e.~\(u(x, t) = u(x + L, t)\) for some domain length \(L\), so that the PDE is
translation-invariant.
Two PDEs of the form~\eqref{eq:generic-1D-pde} are used: \textbf{Burgers' equation} and the \textbf{\ks equation}.
These equations are described in more detail in \ref{sec:data-generation/burgers}
and~\ref{sec:data-generation/kuramoto-sivashinsky}, respectively.

In order to generate training data from these equations, they are discretised using the finite volume method with a
large number of finite volumes.
The resulting ODEs are solved, and the solutions are down-sampled by averaging the resulting vectors \(\vec{u} (t)\) to
obtain a coarser discretisation of the original PDE.
More details regarding the data generation are given in \ref{sec:data-generation}.

\subsection{Burgers' equation}
\label{sec:numerical-experiments/burgers}
The first test equation is Burgers' equation:
\begin{align*}
    \pd{u}{t} &= -\frac{1}{2}\pd{}{x}\left(u^2\right) + \nu \pd[2]{u}{x},
\end{align*}
with \(\nu = 0.0005\), for \(x \in [0, 1]\) with periodic boundary conditions.
Solutions of this PDE contain waves that travel either left or right through the domain (depending on the sign of
\(u\)), resulting in shock waves that are then dissipated.
Training data is generated by discretising the PDE into 4096 finite volumes, solving the resulting ODE for \(t \in [0,
0.5]\) using the standard fourth-order accurate Runge-Kutta method, and down-sampling the solutions to \(N_x = 64\) finite volumes.
The resulting training data consists of \(N_p = 96\) trajectories used for training and 32 for testing.
Each trajectory consists of an initial condition followed by 64 additional snapshots with a time interval of
\(\Delta t = 2^{-7}\) between snapshots, meaning the total number of snapshots per trajectory is \(N_s = 65\).
More information about the data generation is given in \ref{sec:data-generation/burgers}.

Three ML models are trained on data obtained from Burgers' equation.
These are closure models of the form \(\d{\vec{u}}{t} = f(\vec{u}) + \text{NN}(\vec{u}; \vartheta)\).
For all three models, the neural network is a Convolutional Neural Network (CNN) with two convolutional layers; more
details about the neural network are given in \ref{sec:training-details/nn-architectures}.
The convolutional structure is chosen to ensure that the models satisfy the translational invariance present in Burgers' equation.
The models are trained in the three ways outlined in \Cref{sec:gradient-forms}.

\subsubsection{Derivative fitting}
One neural closure model is trained using derivative fitting, meaning the loss function is as given in
\eqref{eq:derivative-fitting}.
The training data consists of \(N_pN_s = 6240\) input-output pairs \(\left(\uref, \left(\d{\vec{u}}
{t}\right)_{\textnormal{ref}}\right)\) on which the model is trained.
This model is trained for 10000 epochs with a batch size of 64.
No regularisation term is added to the loss function as this not is found to meaningfully improve the resulting model.
The choice not to add a regularisation term is motivated in more detail in
\Cref{sec:numerical-experiments/burgers/results}.

\subsubsection{\Discopt}
For the \discopt approach, the neural closure model is embedded in \texttt{Tsit5}, a fourth-order ODE solver due to
Tsitouras \cite{tsitouras2011runge}, with fixed time step.
The loss function is given by equation \eqref{eq:trajectory-fitting} where \(N_p = 96\), \(N_t = 64\), and \(t_i = i\Delta t\)
with \(\Delta t = 2^{-7}\).
In words, the loss function is the mean-square error of the trajectory prediction, averaged over all snapshots of the
training data except the initial condition.
This model is trained for 20000 epochs, since trajectory fitting is found to converge more slowly than derivative fitting.
The training is done with a smaller batch size of 8, since an input-output-pair that can be used for training is now one
of the 96 trajectories, instead of one of the 6240 snapshots.

\subsubsection{\Optdisc}
A third model is trained using \optdisc.
The ODE solver, loss function, batch size, and number of epochs are the same as those for the \discopt model.
The main difference is that for the \optdisc model, the time step of the ODE solver is not fixed, but is determined by
the solver's internal error control mechanism.
To compute the gradients required for training, the adjoint ODE is solved using the same ODE solver as the forward ODE.

\subsubsection{Results}
\label{sec:numerical-experiments/burgers/results}
After training, all models are evaluated on test data that is not used during training.
For each of the evaluated models, \(32\) initial conditions \(\vec{u}^{(j)}_{\text{ref}}(0)\) for \(j = 1, \dots, 32\)
are given to the ML model.
The model then uses these initial conditions to make predictions for the trajectories \(\vec{u}^{(j)}_{\text{predict}}
(t_i)\) for \(j = 1, \dots, 32\) and \(i = 1, \dots, N_t\).
These are then compared to the actual trajectories \(\vec{u}^{(j)}_{\text{ref}}(t_i)\) by taking the root-mean-square
error (RMSE), summing over all components of the vector \(\vec{u}\), all time points \(t_i\), and all 32 testing
trajectories.
Therefore, the RMSE is simply the square root of the expression in \eqref{eq:trajectory-fitting}, except over a
different set of trajectories.
In addition to the neural closure models, the RMSE is also computed when solving the coarse discretisation of Burgers'
equation without a closure term, i.e.~with \(\text{NN}(\vec{u}; \vartheta) \equiv 0\).
The RMSEs on the test data for all models are shown in \Cref{table:burgers-results}.
\begin{table}
    \centering
    \caption{
        The RMSE for each of the tested models on the 32 testing trajectories of the Burgers' equation.
    }
    \label{table:burgers-results}
    \vspace{1em}
    \begin{tabular}{l r}
        \toprule
        Training approach & Error on test data \\
        \midrule
        Coarse ODE without closure term & 0.104\phantom{0} \\
        \midrule
        Derivative fitting              & 2.67\phantom{00} \\
        \Discopt                        & 0.0264 \\
        \Optdisc                        & 0.0312 \\
        \bottomrule
    \end{tabular}
\end{table}
While the two models trained on trajectory fitting achieve significantly lower error than the coarse ODE solved without
a closure term, the model trained using derivative fitting produces highly inaccurate results.
As is visible in \Cref{fig:burgers-derivfit-model-error}, the model trained on derivative fitting produces a trajectory
prediction that has completely different behaviour to the training data.
The error in the prediction grows steadily as \(t\) increases, and eventually reaches a steady-state with high error.
This indicates that the derivative-fitting trained model creates a prediction that converges to a significantly
different steady state to the steady state of the reference solution.
To investigate this result further, more models are trained using derivative fitting with L2-regularisation, and indeed adding
a regularisation term to the loss function does improve the accuracy of the resulting models.
Nevertheless, these results are not included here.
The reason for this is that adding a regularisation term only helps by reducing the magnitude of the neural network
parameters, thereby also reducing the magnitude of the neural network output.
The result is that training with a strong regularisation brings the accuracy of the neural closure model closer to that
of the coarse ODE, without surpassing the coarse ODE in accuracy.
Therefore, even with a regularisation term, derivative fitting is not effective at producing accurate neural closure
models.
Note that the prediction made by the model trained by \discopt, which is shown in \Cref{fig:burgers-discopt-model-error},
has lower error for large \(t\) than for small \(t\).
The reason for this is that solutions to Burgers' equation show
relatively complex transient behaviour before converging to a simple steady state, making the long-term behaviour easier
to predict than the short-term behaviour.

\begin{figure}
    \centering
    \begin{subfigure}[b]{0.48\textwidth}
        \includegraphics[width=\textwidth]{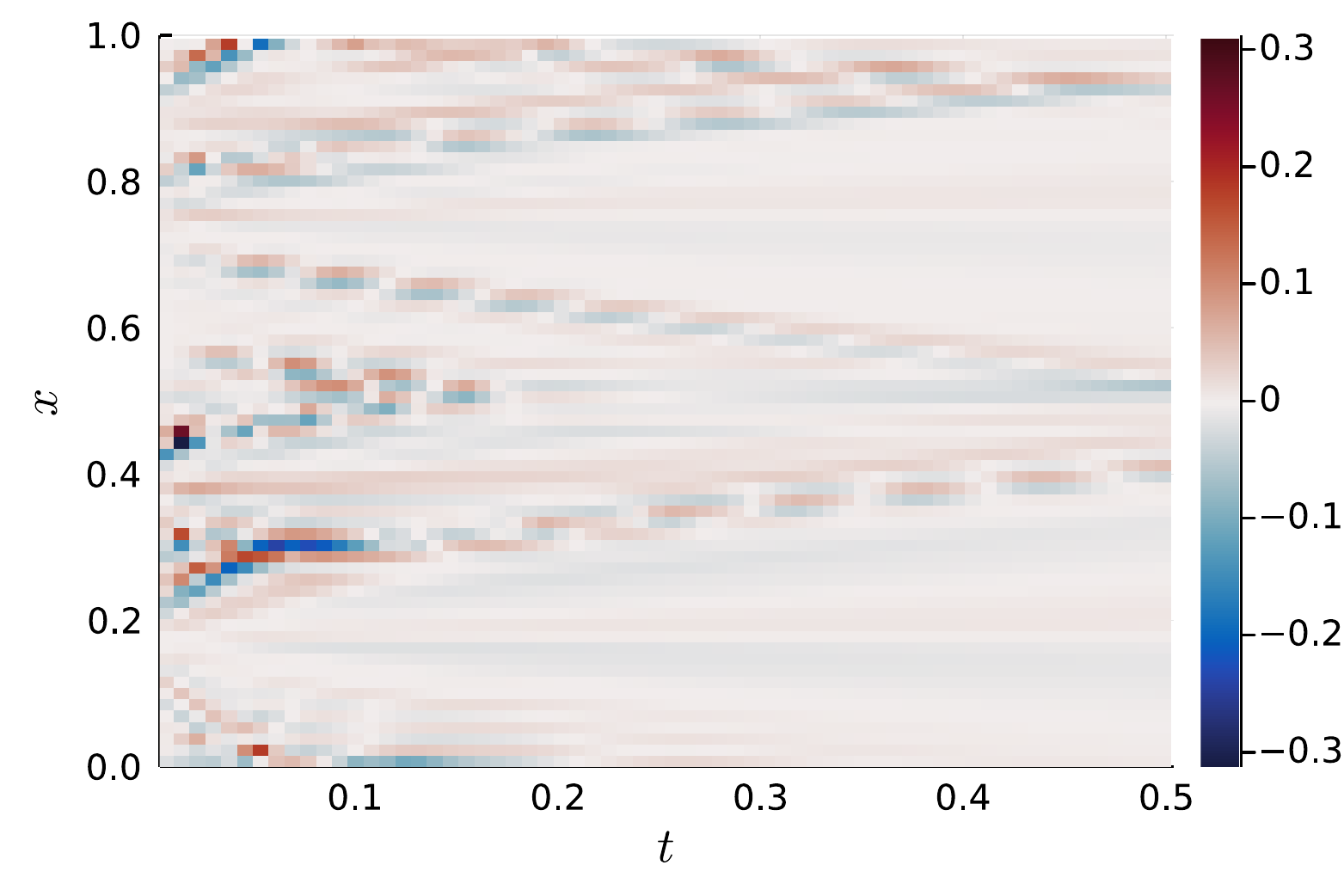}
        \caption{Prediction error by the model trained by \discopt}
        \label{fig:burgers-discopt-model-error}
    \end{subfigure}
    \hfill
    \begin{subfigure}[b]{0.48\textwidth}
        \includegraphics[width=\textwidth]{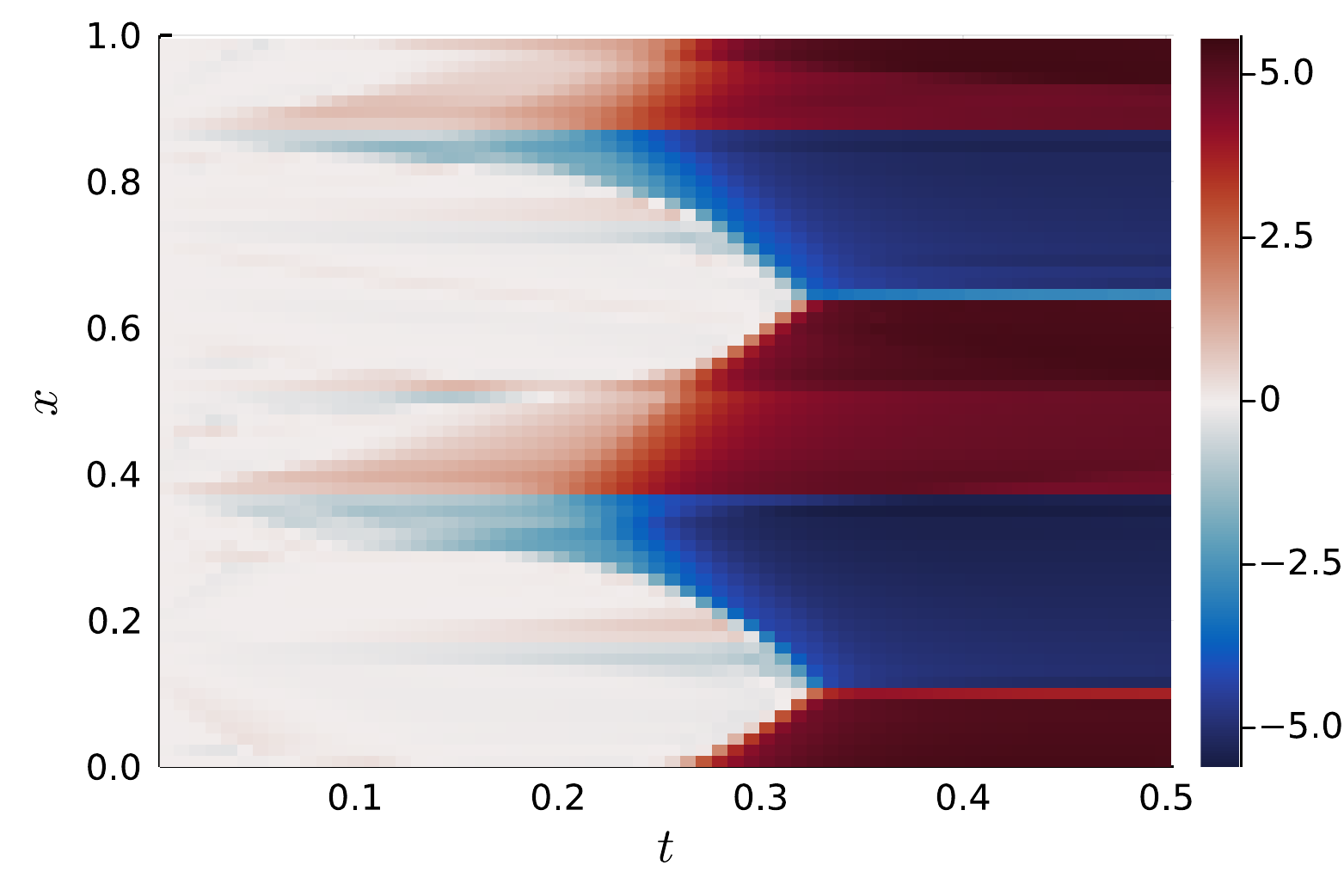}
        \caption{Prediction error by the model trained by derivative fitting}
        \label{fig:burgers-derivfit-model-error}
    \end{subfigure}
    \caption{
        The errors in the trajectory predictions made by the models trained by \discopt and derivative fitting,
        respectively.
    }
    \label{fig:burgers-example-predictions}
\end{figure}

\Cref{fig:burgers-trajfit-losses} shows how the RMSE on the training data of the two trajectory fitting models improves
during training.
From this figure, it is clear that while both models start out with approximately the same RMSE (slightly above 0.1),
the model trained with \discopt improves slightly faster and starts converging to a measurably lower error than the
other model after approximately 8000 epochs.
The lower training error is not due to overfitting, as the model trained with \discopt also produces measurably more
accurate results on the test data as shown in \Cref{table:burgers-results}.
Instead, the reason for the higher training error in \optdisc is due to the inaccuracy of the gradient of the loss
function: while the exact minimiser of the optimisation problem, i.e.~the optimal parameters \(\vartheta\), satisfy
\(\d{\text{Loss}}{\vartheta} = 0\), the gradient of the loss function is not computed exactly when using the adjoint
ODE formulation.
As a result, training by \optdisc does not converge to the truly optimal parameters or even to a local optimiser, but to
a different parameter vector that satisfies \(\d{\text{Loss}}{\vartheta} + (\text{adjoint ODE error}) = 0\).

\begin{figure}
    \centering
    \includegraphics[width=0.48\textwidth]{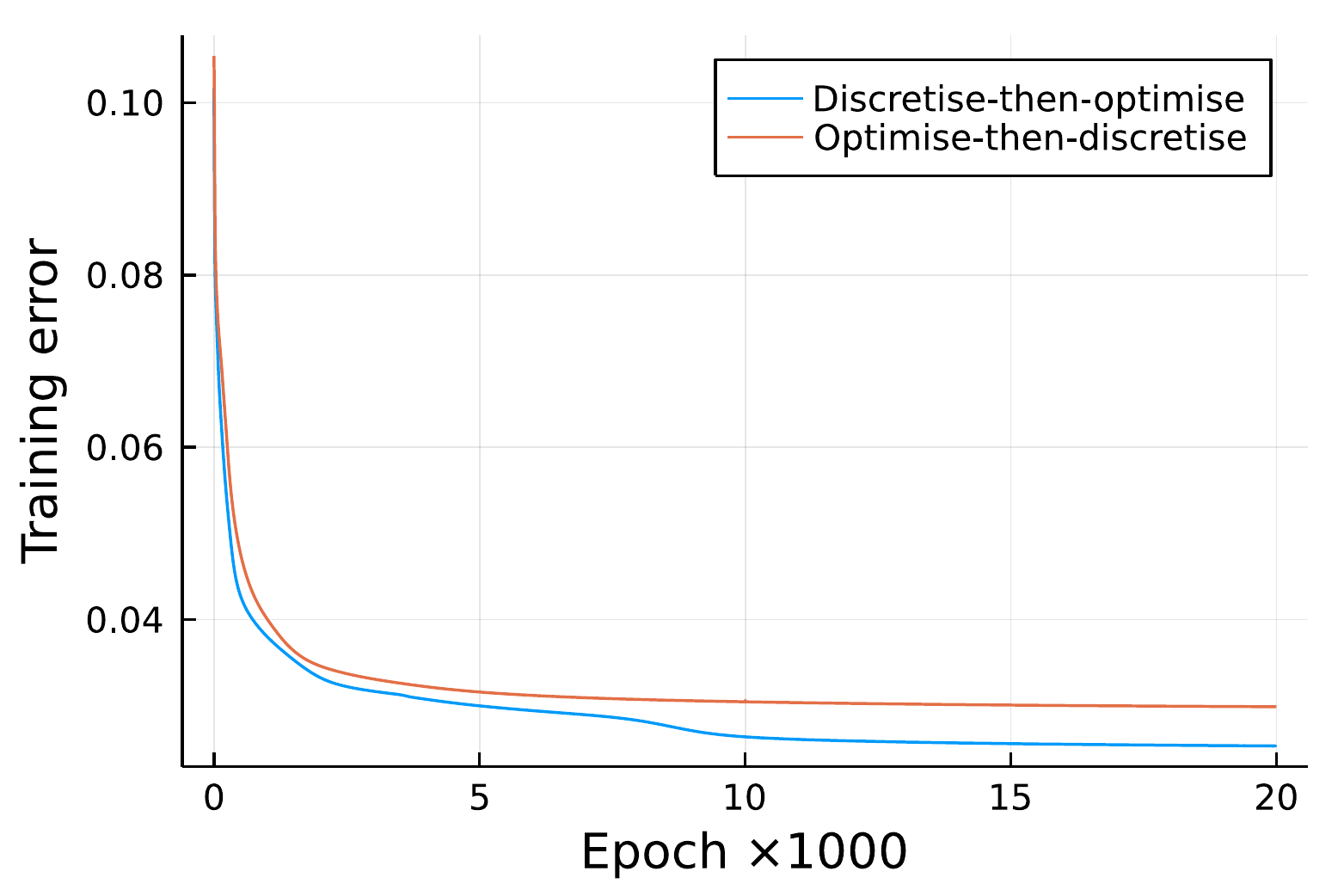}
    \caption{
        The RMSE history during training of the two models that are trained using trajectory fitting on the Burgers'
        test case.
    }
    \label{fig:burgers-trajfit-losses}
\end{figure}

\subsection{The \ks equation}
\label{sec:experiments/ks}
Since Burgers' equation is relatively predictable (resulting in shock waves which then dissipate over time while moving
through the domain), we also consider a more challenging case.
In this section, experiments are performed using the \ks equation:
\begin{align*}
    \pd{u}{t} &= -\frac{1}{2}\pd{}{x}\left(u^2\right) - \pd[2]{u}{x} - \pd[4]{u}{x},
\end{align*}
for \(x \in [0, 64]\) with periodic boundary conditions.
For a short description of the \ks equation and its terms and solution behaviour, see \ref{sec:data-generation/kuramoto-sivashinsky}.
Training data for the \ks equation was generated by discretising the PDE into 1024 finite volumes, solving the resulting
ODE using a third-order stiff ODE solver for \(t \in [0, 256]\), and down-sampling the resulting solution to \(N_x =
128\) finite volumes.
The resulting training data consists of \(N_p = 80\) trajectories used for training and 10 for testing.
Each trajectory consists of an initial state followed by \(N_t = 512\) additional snapshots with \(\Delta t = \frac{1}
{2}\) between snapshots.
All models trained in this section use the `large' neural network shown in \Cref{table:large-nn}, which is a CNN with
six convolutional layers.

The \ks equation is chaotic, meaning that arbitrarily small differences in the initial state \(\vec{u}(0)\) will
eventually lead to a completely different solution \(\vec{u}(t)\) for large \(t\).
As a result, all methods of approximating this PDE are expected to diverge from the training data at some point, meaning
that simply taking the RMSE between the predicted and actual trajectories is not a very useful metric for accuracy.
Instead, the accuracy of a method will be evaluated using the valid prediction time (VPT), which is the time until the
error between the approximation and the training data exceeds some pre-defined threshold.
Here, the VPT of a prediction \(\vec{u}(t)\) to a real trajectory \(\uref(t)\) is computed following the procedure used
by Pathak \etal\cite{pathak2018hybrid}.
First, the average energy of the real trajectory is computed as
\begin{align*}
    E_{\text{avg}} &= \sqrt{\frac{1}{N_t} \sum_{i = 1}^{N_t} \norm{\uref(t_i)}^2}.
\end{align*}
Then, the valid prediction time is given by
\begin{align}
    \label{eq:vpt-definition}
    \text{VPT}(\uref, \vec{u}, t_{1, 2, \dots, N_t})
    &= \min\left\{ t_i~|~\norm{\vec{u}(t_i) - \uref(t_i)} \geq 0.4E_{\text{avg}} \right\}.
\end{align}

One property shared among chaotic processes is that the difference between two solutions \(\uref, \vec{u}\) grows
approximately exponentially in time:
\begin{align}
    \label{eq:chaos-exponential-trajectory-divergence}
    \norm{\uref(t) - \vec{u}(t)}
    &\approx e^{\lambda_{\text{max}}t} \norm{\uref(0) - \vec{u}(0)}.
\end{align}
In this equation, \(\lambda_{\text{max}}\) is the Lyapunov exponent of the ODE, which determines the growth
rate of the error between two solutions.
The inverse of the Lyapunov exponent is the Lyapunov time \(\tlyap = \lambda_{\max}^{-1}\), which can be
interpreted as the time it takes for the difference between two similar trajectories to increase by a factor \(e\).
Since the Lyapunov time represents the time scale on which trajectories diverge, the valid prediction times of models
will be expressed as multiples of \(\tlyap\).

\subsubsection{Derivative fitting}
\label{sec:experiments/ks/derivfit}
As is the case for Burgers' equation, the simplest training approach available is to train using derivative fitting.
Since adding a penalty term to the neural network is not found to help significantly, this penalty term is omitted for
derivative fitting for the \ks equation, meaning only one model is trained with derivative fitting.
This model is trained on all snapshots of the 80 trajectories of the training data.
The model is trained for 1000 epochs with the batch size set to 128.

\subsubsection{\Optdisc}
\label{sec:experiments/ks/optdisc}
The \optdisc training approach is tested on the same neural network.
With this approach, the parameter \(N_t\), the number of snapshot predictions computed from the initial condition, must
be chosen.
For the models trained on Burgers' equation, \(N_t\) was equal to the number of snapshots in the training data, but this
is not expected to yield accurate models for the chaotic \ks equation as described in \Cref{sec:theory/trajfit}.
Here, two models are trained with the \optdisc approach with two different choices for \(N_t\).
The first model is trained on the first 25 snapshots from each trajectory (i.e.~an initial condition and \(N_t = 24\)
additional snapshots, corresponding to one Lyapunov time), and the other is trained on the first 145 snapshots from each
trajectory (i.e.~with \(N_t = 144\), corresponding to six Lyapunov times).
This is done so that the effect of the trajectory length \(N_t\) on the model accuracy can also be studied.
For both models, the forward and adjoint ODEs are both solved using \texttt{KenCarp47}, a \nth[4]-order accurate
additive Runge-Kutta method due to Kennedy and~Carpenter \cite{kennedy2019higher} that is implicit in \(f\) but explicit
in the neural network term.
Using this ODE solver was found to be computationally more efficient than using either a fully explicit or fully
implicit Runge-Kutta method (see sections 5.5.3 and 5.6.1 of \cite{mastersthesis}).

In order to avoid very large gradients causing the training to fail, gradient clipping is applied to the
gradients before applying the optimiser.
This way, gradients of which the norm exceeds some constant \(r\) are scaled such that their norms are exactly equal to
\(r\), thereby avoiding very large gradients while leaving small gradients unchanged.
Here, gradient clipping is used with \(r = 10^{-2}\).
The model trained on short trajectories is trained for 1000 epochs with a batch size of 10.
The model trained on long trajectories is only trained for 100 epochs since the longer trajectories make the training
much slower.
While this means that the second model is not trained until convergence, the effect on the accuracy of the resulting
models is found to be small compared to the difference in accuracy reported in \Cref{sec:experiments/ks/results}.
Both models are trained with a batch size of 8.

As described in \Cref{sec:theory/trajfit}, trajectory fitting on chaotic systems can cause problems due to exploding
gradients, especially when training on long trajectories.
To see if weighing the loss function as in \eqref{eq:exp-weighted-mse} mitigates this problem, four more models are
trained using \optdisc on long trajectories, using the weighted MSE as loss function with \(c \in \{0.5, 1.0, 1.5,
2.0\}\).
The number of epochs, batch size, and other parameters for these models is the same as those for the model trained on
long trajectories with a simple MSE loss function.

\subsubsection{\Discopt}
\label{sec:experiments/ks/discopt}
To test the \discopt method for the \ks equation, the PDE is solved in the pseudospectral domain using an exponential
integrator.
More information about this approach is given in \ref{sec:ks-pseudospectral}.

In order to be able to compare more directly with the previous models, the closure term is given by the same neural
network as used in earlier experiments, meaning that its input and output are in the physical domain.
As such, the neural network term is preceded by an inverse Fourier transform and followed by a Fourier transform:
\begin{align}
    \label{eq:ks-pseudo-spectral-with-closure}
    \d{}{t}\hat{\vec{u}}
    &= \left( \mat{\Lambda}^2 - \mat{\Lambda}^4 \right)\hat{\vec{u}}
        - \frac{i}{2}\mat{\Lambda}\mathcal{F}\left( \left( \mathcal{F}^{-1}\hat{\vec{u}} \right)^2 \right)
        + \mathcal{F}\left( \text{NN}\left( \mathcal{F}^{-1}\hat{\vec{u}}; \vartheta \right) \right).
\end{align}

As is the case for the \optdisc tests, one must choose \(N_t\), the number of time steps that are predicted with the
model.
In \Cref{sec:numerical-experiments/burgers}, the number of time steps to predict was 64, equal to the number of
snapshots in the training data, with good results.
In their experiments with neural closure models for two-dimensional incompressible fluid flow problems, List
\etal\cite{list2022learned} refer to this as the number of unrolled steps, and find that this has a significant effect
on the stability of the resulting model.
As such, a number of different models are trained by unrolling with different numbers of time steps.
Specifically, 11 different models are trained with \(N_t \in \left\{ 1, 2, 4, 8, 15, 30, 60, 90, 100, 110, 120
\right\}\).
The model with \(N_t = 120\) is trained on the first 121 snapshots of each of the 80 training trajectories.
For the models with \(N_t \in \left\{ 90, 100, 110 \right\}\), the same 80 trajectories are truncated to the desired
number of steps.
For the remaining models, the training trajectories are split into multiple shorter trajectories as shown in
\Cref{fig:trajectory-splitting-example}.
In this way, each training trajectory of 121 snapshots (an initial condition followed by 120 additional time steps) can
be split into 4 trajectories of 31 snapshots each, or 15 trajectories of 9 snapshots each, and so on.

\begin{figure}
    \centering
    \begin{subfigure}[b]{0.48\textwidth}
        \includegraphics[width=\textwidth]{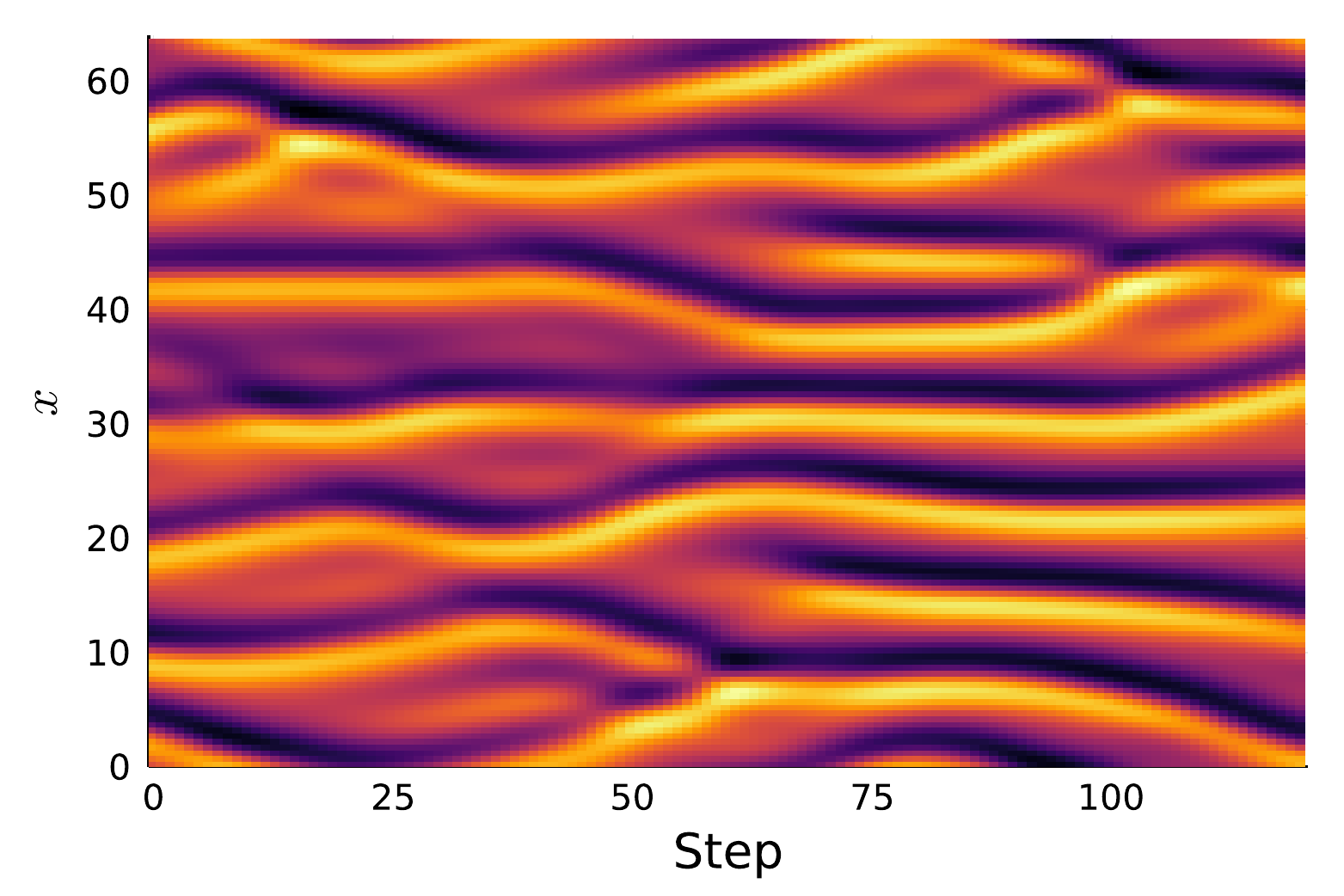}
        \caption{A `long' trajectory consisting of the initial state and 120 additional snapshots.}
    \end{subfigure}
    \hfill
    \begin{subfigure}[b]{0.48\textwidth}
        \includegraphics[width=\textwidth]{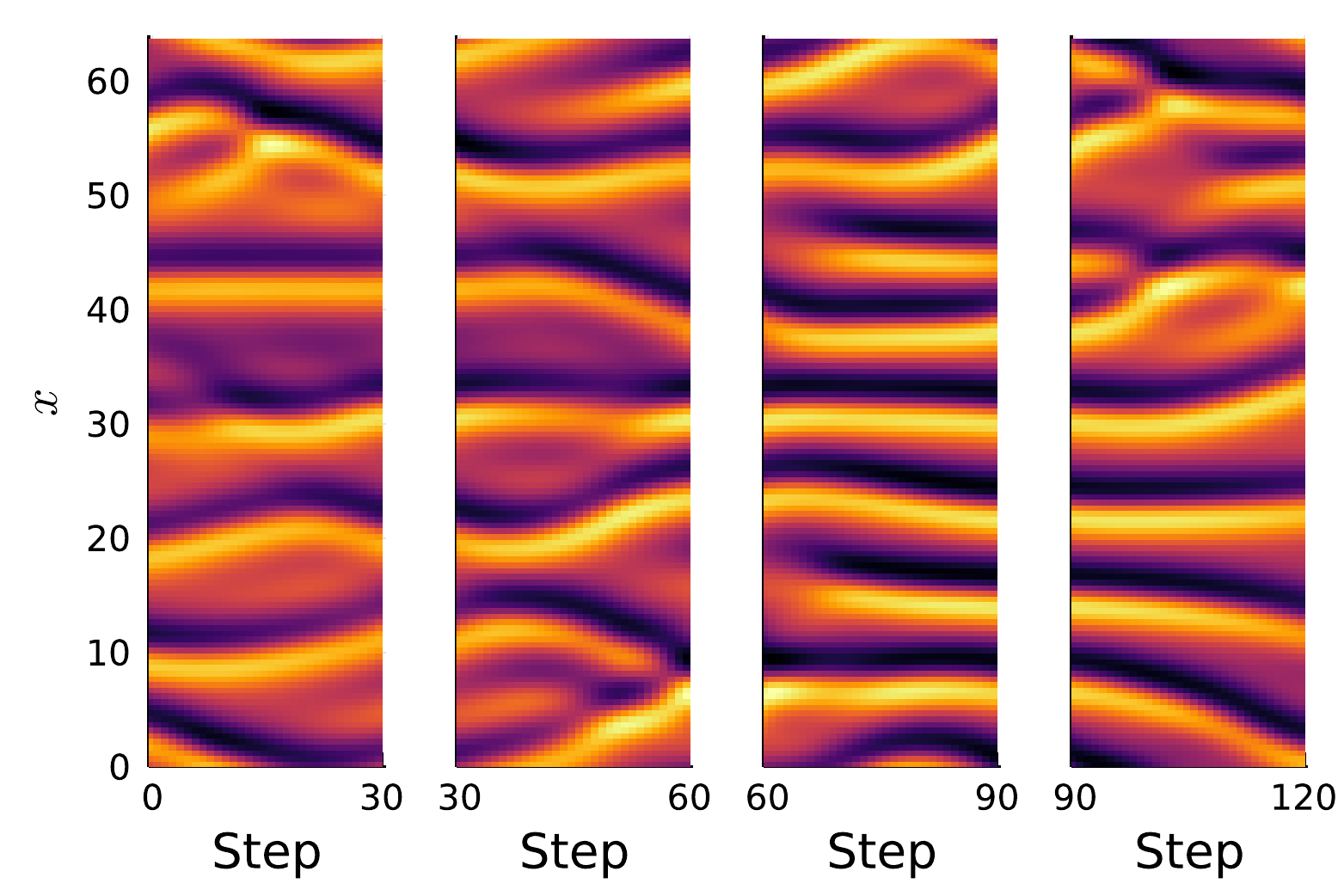}
        \caption{Four `short' trajectories, each consisting of an initial condition and 30 additional snapshots.}
    \end{subfigure}
    \caption{
        A visual representation of how training data from the \ks equation is split into smaller trajectories.
        The last snapshot of one trajectory is equal to the initial condition of the next trajectory.
    }
    \label{fig:trajectory-splitting-example}
\end{figure}

This way, all \discopt models are trained on the first 121 snapshots of each trajectory in the training data (or
slightly fewer).
All these models are trained for 5000 epochs with a batch size of 8.

\FloatBarrier

\subsubsection{Results}
\label{sec:experiments/ks/results}
\Cref{table:ks-main-results} shows the minimum, average, and maximum VPTs of neural ODEs trained with derivative fitting
and \optdisc, and some example trajectories and predictions are shown in \Cref{fig:ks-some-results}.
Note that the VPTs in \Cref{fig:ks-trajfit-short-prediction,fig:ks-trajfit-long-prediction} only show one of the ten test trajectories, meaning that the VPTs in these figures do not correspond to the averages shown in \Cref{table:ks-main-results}.
From this table,
it can be seen that \optdisc fitting works best on short trajectories, where the resulting models slightly outperform
models trained by derivative fitting.
Training on long trajectories results in worse performance, even if the loss function is exponentially weighted as in \eqref{eq:exp-weighted-mse} to mitigate the
exploding gradients problem.
This is likely due to the fact that the adjoint ODE methods introduce an error in the gradients used during training.
As is generally the case for ODE solutions, this error increases with the time interval over which the ODE is solved.
Hence, training on longer trajectories introduces a larger error in the gradients, which reduces the accuracy of the
resulting model even if the error function is weighted to mitigate the exploding gradients problem.

\begin{table}[h]
    \centering
    \caption{
        An overview of the performance of different models tested on the \ks equation, sorted by training approach.
        \Cref{fig:ks-results-overview} summarises all results for the \ks equation.
    }
    \label{table:ks-main-results}
    \vspace{1em}
    \begin{tabular}{l l r r r}
        \toprule
        \multicolumn{2}{l}{\multirow{2}{*}{Training approach}}        & \multicolumn{3}{c}{VPT on test data} \\
        \cmidrule{3-5}
        \multicolumn{2}{l}{}                                          & Min  & Avg  & Max  \\
        \midrule
        \multicolumn{2}{l}{Coarse ODE}                                & 1.17 & 1.93 & 3.00 \\
        \midrule
        \multicolumn{2}{l}{Derivative fitting}                        & 4.17 & 5.36 & 7.54 \\
        \midrule
        \multirow{2}{*}{\Optdisc} & Short trajectories (\(N_t = 24\)) & 4.08 & 5.84 & 8.29 \\
                                  & Long trajectories (\(N_t = 144\)) & 2.38 & 3.38 & 4.67 \\
        \midrule
                                  & \(c=0.5\)                         & 2.42 & 4.20 & 5.38 \\
        \Optdisc, \(N_t=144\)     & \(c=1.0\)                         & 2.96 & 4.38 & 6.29 \\
        decaying error weights    & \(c=1.5\)                         & 3.29 & 4.58 & 5.88 \\
                                  & \(c=2.0\)                         & 2.71 & 4.29 & 5.75 \\
        \midrule
        \multirow{2}{*}{\Discopt} & Short trajectories (\(N_t = 30\)) & 4.92 & 7.10 & 9.12 \\
                                  & Long trajectories (\(N_t = 120\)) & 4.12 & 5.33 & 7.38 \\
        \bottomrule
    \end{tabular}
\end{table}

% ETDRK4
\Cref{fig:ks-pseudospectral-vpt-chart} shows the minimum, average, and maximum VPTs of the 11 models trained using
\discopt on different numbers of unrolling steps.
It can be seen that models trained on \(60\) or more steps perform worse than models trained on fewer steps, although
performance continues to improve after 1000 epochs.
After 5000 epochs, the model trained on 30 steps performs the best.
The model trained on 120 steps is found to perform very badly after 1000 epochs, although performance is similar to that
of other models after 5000 epochs.
The reason for the poor performance after 1000 epochs is described in \Cref{sec:theory/trajfit}: the long time interval
used for training, combined with the chaotic nature of the \ks equation, means that the loss function is most sensitive
to the behaviour for large \(t\).
This initially prevents the model from becoming more accurate in the short term, resulting in a very short VPT.
\iffalse
\begin{table}
    \centering
    \caption{
        The minimum, average, and maximum valid prediction times of each of the eleven models trained by \discopt, after
        1000 and 5000 epochs.
        The same data is shown in \cref{fig:ks-pseudospectral-vpt-chart}.
    }
    \label{table:ks-pseudospectral-vpt-chart}
    \vspace{1em}
    \begin{tabular}{r | c | c | c | c | c | c}
        \multirow{3}{*}{\(N_t\)} & \multicolumn{6}{c}{VPT} \\
        & \multicolumn{3}{c |}{1000 epochs} & \multicolumn{3}{c}{5000 epochs} \\
        & Min & Avg & Max & Min & Avg & Max \\
        \hline
        1 & 4.42 & 6.00 & 7.96 & 4.62 & 6.10 & 7.92 \\
        2 & 5.17 & 6.42 & 7.92 & 5.12 & 6.70 & 8.96 \\
        4 & 5.04 & 6.55 & 8.33 & 4.62 & 6.63 & 8.42 \\
        8 & 4.58 & 6.47 & 8.29 & 4.38 & 6.18 & 7.42 \\
       15 & 5.04 & 6.35 & 8.38 & 4.88 & 6.95 & 8.54 \\
       30 & 4.38 & 6.32 & 8.79 & 4.92 & 7.10 & 9.12 \\
       60 & 3.62 & 4.97 & 6.83 & 3.62 & 5.01 & 6.25 \\
       90 & 3.67 & 5.07 & 7.17 & 4.38 & 5.42 & 6.88 \\
      100 & 2.79 & 4.18 & 5.46 & 3.54 & 5.05 & 8.71 \\
      110 & 3.42 & 4.38 & 5.42 & 3.54 & 5.00 & 6.29 \\
      120 & 0.25 & 0.28 & 0.33 & 4.12 & 5.33 & 7.38
    \end{tabular}
\end{table}
\fi

\begin{figure}
    \centering
    \begin{subfigure}[b]{0.48\textwidth}
        \includegraphics[width=\textwidth]{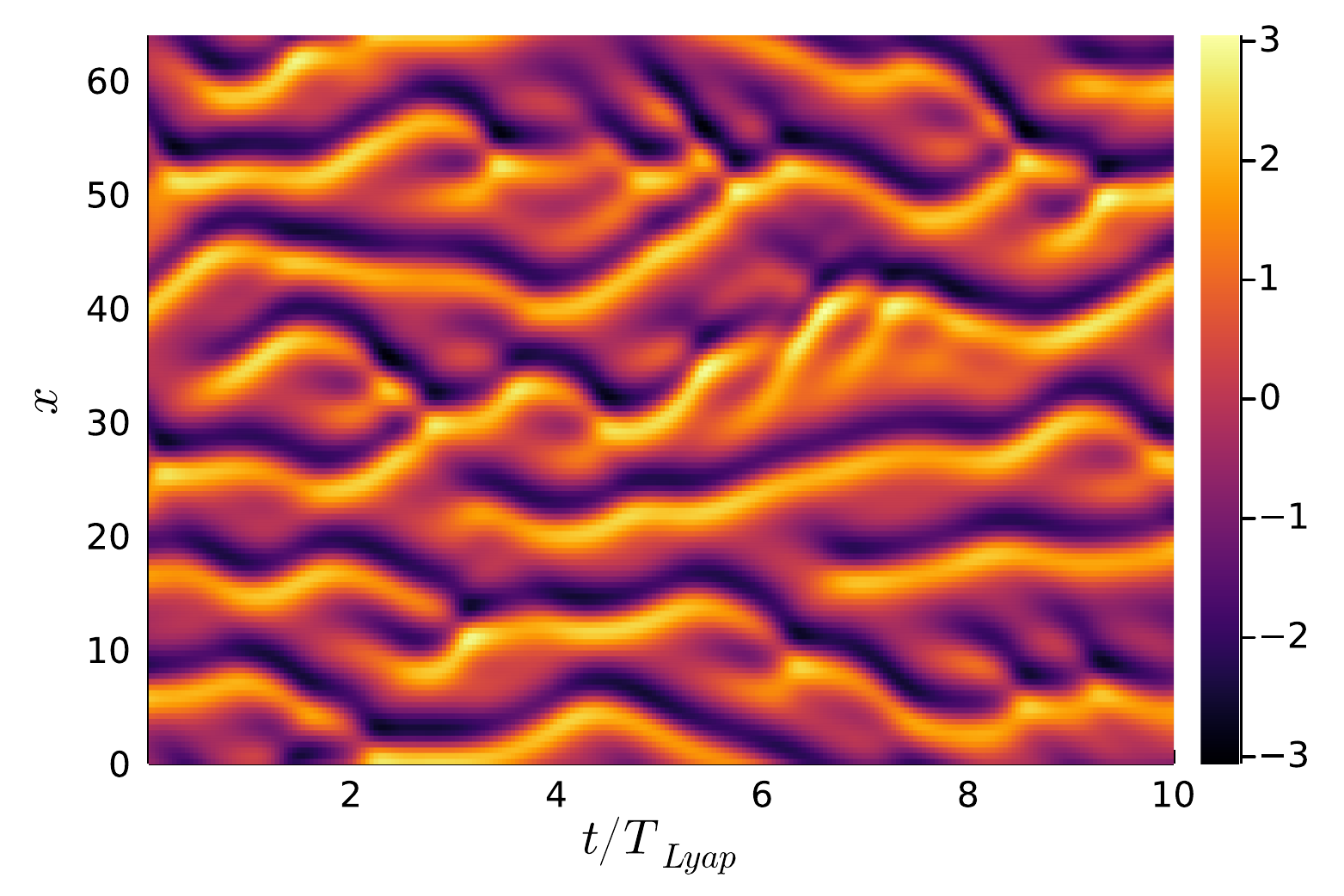}
        \caption{A trajectory of the K-S equation from the test data.}
        \label{fig:ks-actual-trajectory}
    \end{subfigure} \\
    \begin{subfigure}[b]{0.48\textwidth}
        \includegraphics[width=\textwidth]{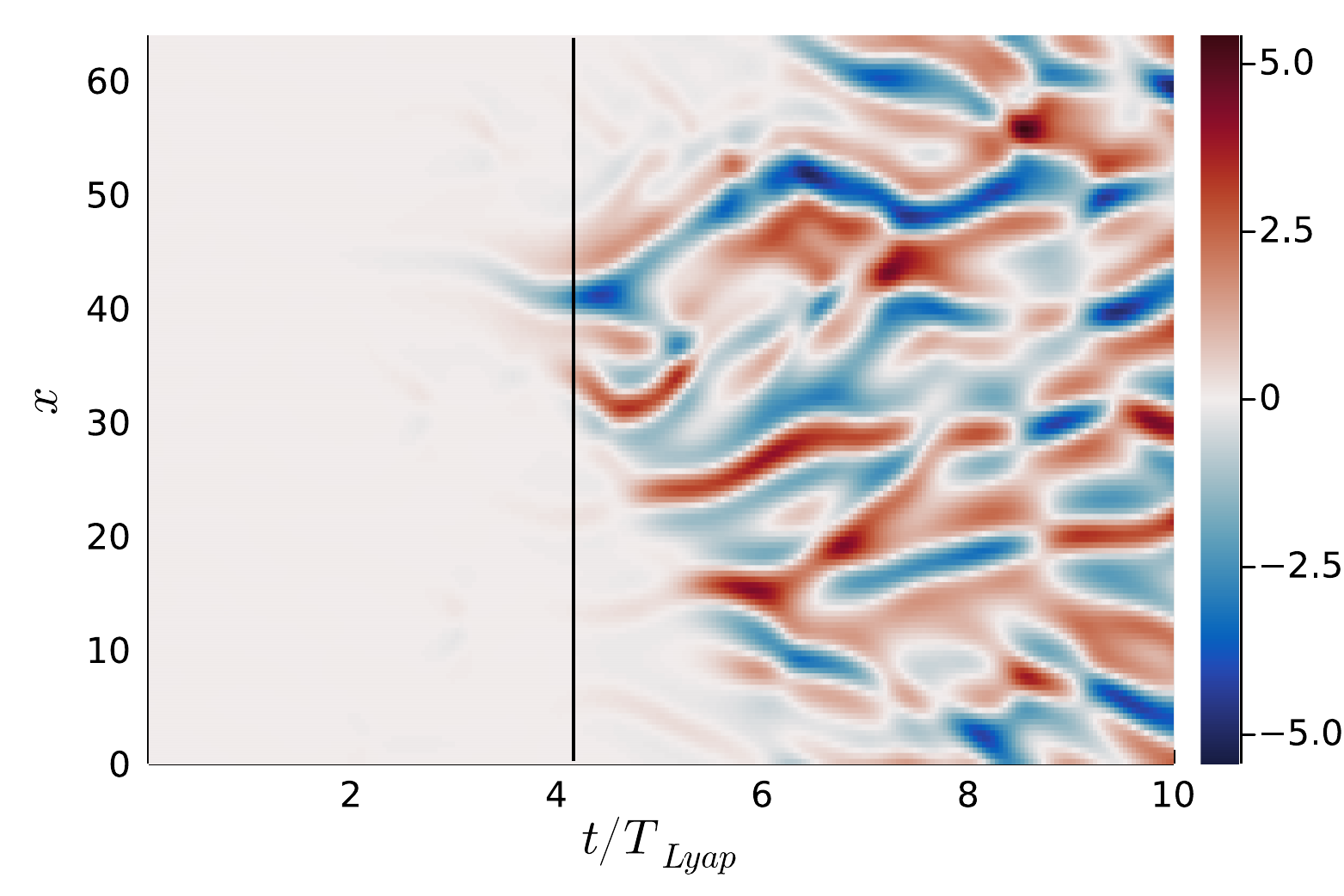}
        \caption{The prediction error for this trajectory of the closure model trained on short trajectories.}
        \label{fig:ks-trajfit-short-prediction}
    \end{subfigure}
    \hfill
    \begin{subfigure}[b]{0.48\textwidth}
        \includegraphics[width=\textwidth]{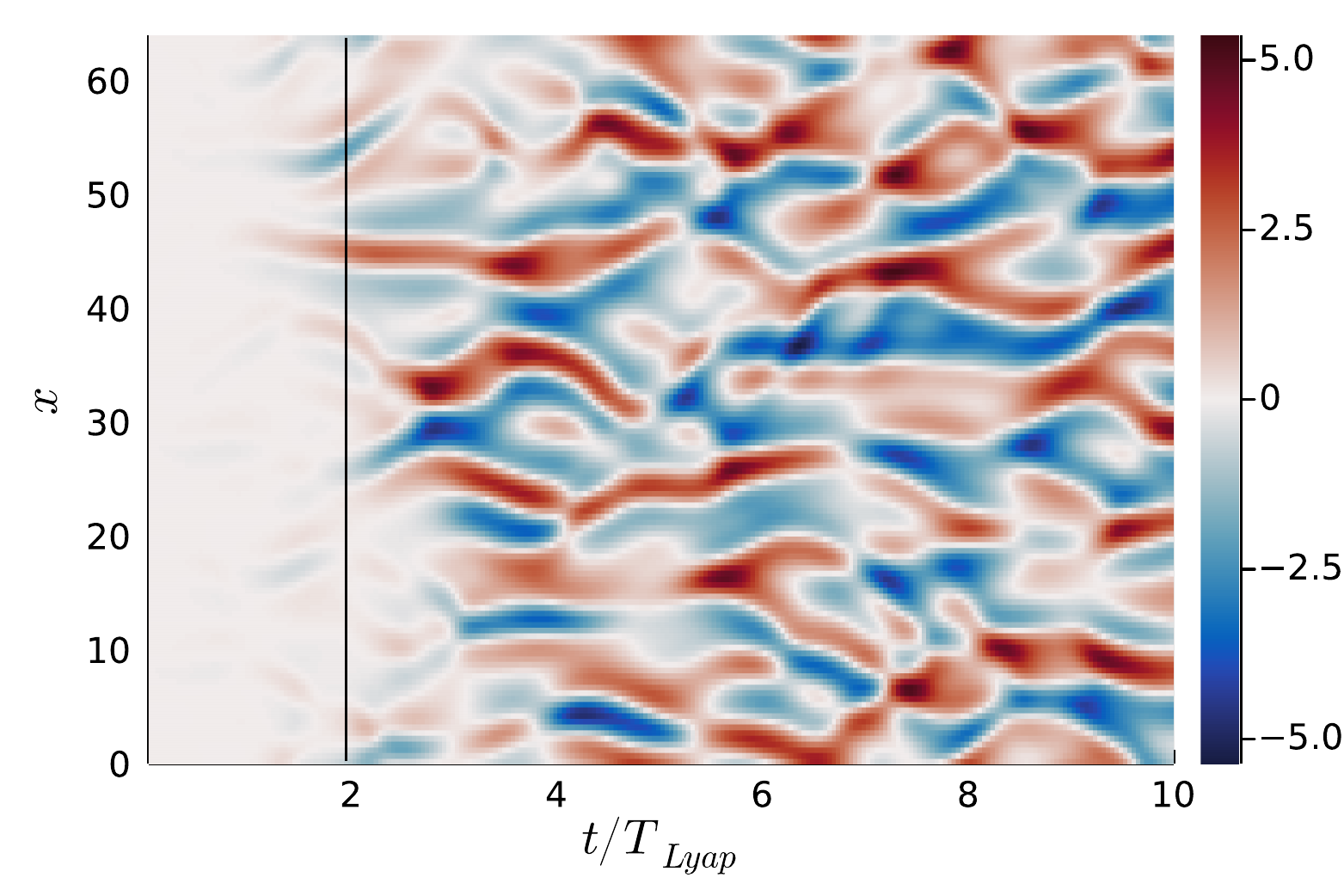}
        \caption{The prediction error for this trajectory of the closure model trained on long trajectories.}
        \label{fig:ks-trajfit-long-prediction}
    \end{subfigure}
    \caption{
        An example trajectory from the test data, and the prediction error of two different models obtained for this
        trajectory.
        Notice how the model trained on long trajectories produces significantly greater error in the short-term, and
        also achieves a lower VPT for this trajectory as a result (indicated by the vertical black lines in
        \Cref{fig:ks-trajfit-short-prediction,fig:ks-trajfit-long-prediction}).
    }
    \label{fig:ks-some-results}
\end{figure}

\begin{figure}
    \centering
    \includegraphics[width=0.70\textwidth]{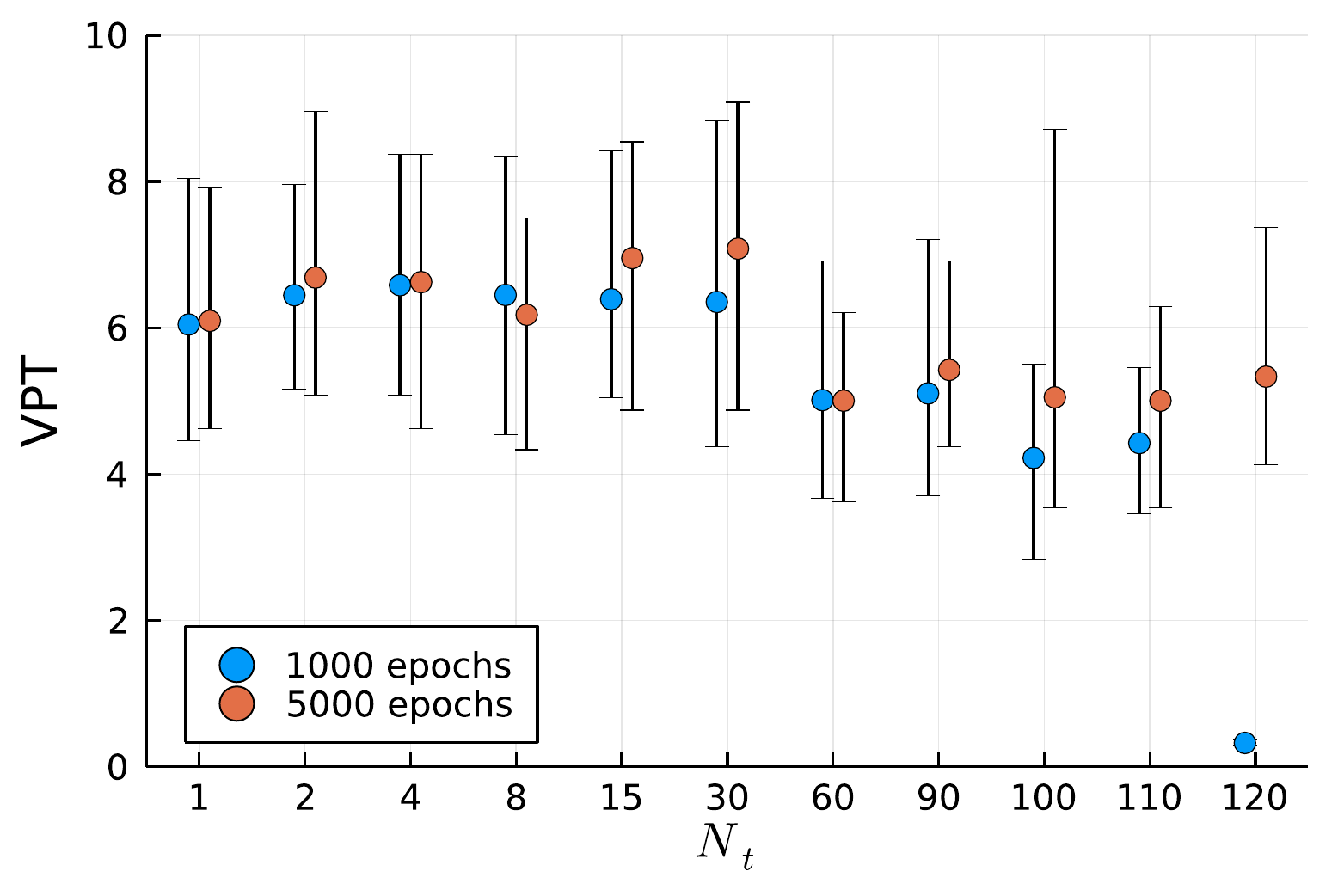}
    \caption{
        The minimum, average, and maximum VPT across the 10 test trajectories for models trained for the given number of
        time steps \(N_t\) on the \ks equation with \discopt.
    }
    \label{fig:ks-pseudospectral-vpt-chart}
\end{figure}

\begin{figure}
    \centering
    \includegraphics[width=0.70\textwidth]{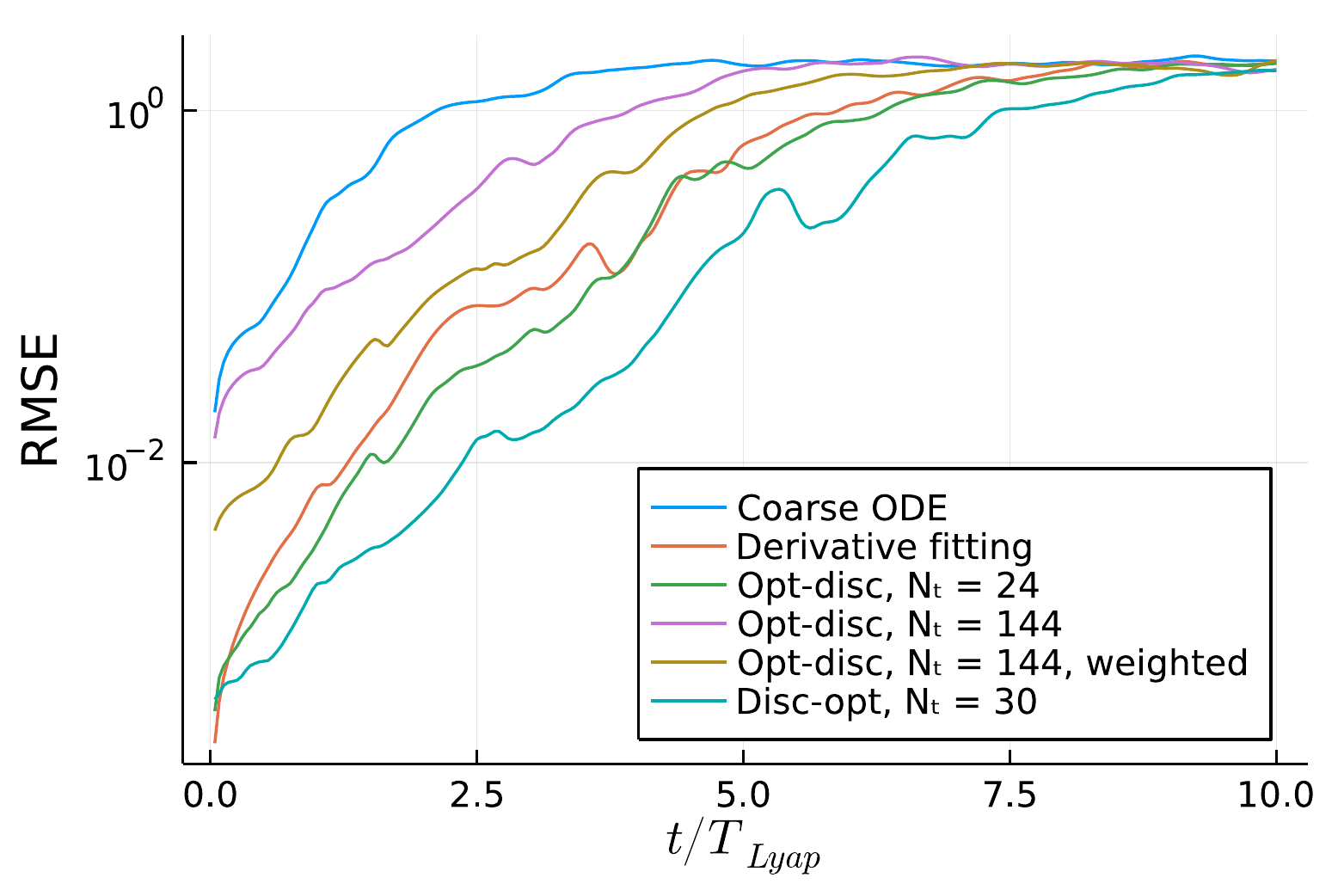}
    \caption{
        A plot of the RMSE on the \ks equation of some of the different models in \Cref{table:ks-main-results} as a function of \(t\)
    }
    \label{fig:ks-results-overview}
\end{figure}

\FloatBarrier

\section{Conclusions}
\label{sec:conclusion}
The goal of this paper is to make a general comparison of approaches for neural closure models, and make recommendations
based on the results of this comparison.
To this end, a variety of different training approaches were evaluated on two different test cases derived from partial
differential equations.
Furthermore, two theorems were given that provide upper bounds on the long-term error of a neural closure model based on
the short-term error that is used as the objective function during training.

The simplest approach, derivative fitting, is found to perform approximately as well as trajectory fitting on the \ks
equation, but performs significantly worse on Burgers' equation.
\Cref{thm:error-bound-continuous} implies that in general, models that are trained with derivative fitting may indeed
produce inaccurate trajectories, even if they achieve high accuracy during training.
It is not entirely clear what causes derivative fitting to perform well on the \ks equation and not on Burgers' equation.
The performance may depend on the ODE for which the neural network acts as a closure term, or on the time integration
method used to solve the ODE.
However, the performance of derivative fitting may also be sensitive to specific parameters of the training procedure.
For example, the use of derivative fitting for LES closure models produces accurate models in the work of Guan
\etal\cite{guan2022learning}, but not in the work of MacArt \etal\cite{macart2021embedded}.
Since derivative fitting is computationally much cheaper than trajectory fitting, it is generally worthwhile to try training a neural closure model with derivative fitting first.
If derivative fitting produces poor models as it does in the Burgers' equation tests, one can still switch to trajectory
fitting.
Furthermore, if derivative fitting gives stable but inaccurate models, another approach is
to train with derivative fitting for a relatively small number of epochs, and to continue training the resulting model with trajectory fitting (see Section 5.1.2 of Melchers~\cite{mastersthesis} for details).
This approach produces accurate models like trajectory fitting, and is computationally more efficient by reducing
the number of iterations of the slower trajectory fitting that have to be performed in order to train the model to
convergence.

Regarding trajectory fitting, the preferable approach is to use the \discopt strategy, rather than \optdisc, which is in
agreement with the results obtained by Onken and Ruthotto~\cite{onken2020discretize}.
The reason for this is likely that the \discopt approach computes gradients more accurately, allowing the model to be
trained to a higher accuracy.
This results in slightly faster convergence during training as well as in a smaller error overall.
This also means that a systematic comparison between \optdisc and \discopt is not always straightforward; in the \ks
case the pseudospectral method used for the \discopt tests is a more accurate spatial discretisation of the PDE than the
finite volume method used in the derivative fitting and \optdisc tests.
This is a fundamental issue in testing training approaches for neural closure models, as many test problems require
specific ODE solvers to efficiently obtain accurate solutions.

When training by trajectory fitting, the length of trajectories used for training must be chosen carefully, both for
\optdisc and \discopt.
This is in line with our \Cref{thm:error-bound-discrete}, stating that models that produce accurate
predictions in the short term may still produce inaccurate predictions in the long-term.
However, this is not always the case as \Cref{thm:error-bound-discrete} only provides an upper bound for the
error.
On the \ks data, for example, the models trained by derivative fitting and by \discopt with unrolling a single time step
both produce good long-term predictions.

For \optdisc, training on long trajectories results in less accurate models due to the exploding gradients problem.
This problem can be mitigated by introducing a weighted error function, although the resulting approach still produces
less accurate models than \discopt trajectory fitting due to the increased error in the gradient computation.
For the \ks test case, derivative fitting produces better results than \optdisc on long trajectories.
For \discopt, fairly accurate models can be obtained by training on long trajectories, although training this way
converges much more slowly than when training on shorter trajectories.

As machine learning techniques become more popular as a way to learn models from data, the existence of general
recommendations and rules of thumb becomes increasingly important to reduce the amount of trial and error required to
obtain accurate models.
Overall, the work presented here provides a set of recommendations for neural closure models.
Notwithstanding, we point to a few remaining open questions.
Most notably, it is not fully clear yet what properties of an ODE system determine whether or not derivative fitting produces
accurate models.
Furthermore, for trajectory fitting it was argued that the number \(N_t\) of time steps computed during training should
be chosen carefully, but in this work good values for \(N_t\) were found by trial and error.
It is not clear how a good value for \(N_t\) can be chosen ahead of time.

\section*{Declaration of competing interests}
The authors report no competing interests.

\section*{Acknowledgements}
Large parts of this work were completed as part of a master thesis on neural closure models at the Department of
Mathematics and Computer Science at Eindhoven University of Technology (TU/e) and at the Centrum Wiskunde \& Informatica
(CWI) in Amsterdam.

\section*{Funding statement}
This publication is part of the project “Discretize first, reduce next" (with project number VI.Vidi.193.105) of the
research programme Vidi which is (partly) financed by the Dutch Research Council (NWO).

\section*{Author contributions}
% CRediT author statement (https://www.elsevier.com/authors/policies-and-guidelines/credit-author-statement)
\textbf{Daan Crommelin}: conceptualisation, writing - review \& editing, supervision, project administration.
\textbf{Barry Koren}: conceptualisation, writing - review \& editing, supervision, project administration.
\textbf{Hugo Melchers}: conceptualisation, methodology, software, formal analysis, investigation, writing - original draft, visualisation.
\textbf{Vlado Menkovski}: conceptualisation, writing - review \& editing, supervision, project administration.
\textbf{Benjamin Sanderse}: conceptualisation, writing - review \& editing, supervision, project administration, funding acquisition.

\bibliography{references}{}
\bibliographystyle{plain}

\appendix

\section{Data generation}
\label{sec:data-generation}
Solutions to the Burgers and \ks equations are computed using the finite volume method.
For both equations, the initial states \(\vec{u}(0)\) are randomly generated as the sum of random sine and cosine waves
with wave numbers \(1 \leq k \leq 10\):
\begin{align}
    \label{eq:random-smooth-initial-state}
    \vec{u}_i(0)
    = \text{Re}\left( \sum_{k = 1}^{10}\hat{u}_k \exp(2\pi i k / N_x)
    + \sum_{k = 1}^{10}\hat{u}_{-k} \exp(-2\pi i k / N_x) \right),
\end{align}
where \(N_x\) is the number of finite volumes and the \(\hat{u}_k, k = \pm 1, \dots, \pm 10\) are independent and
distributed according to a unit Gaussian distribution.
The resulting initial conditions are then multiplied by a constant so that \(\max_i \abs{\vec{u}_{i}(0)} = 2\).

\subsection{Burgers' equation}
\label{sec:data-generation/burgers}
Burgers' equation is a PDE over one variable, the momentum \(u(x, t)\), as a function of space and time.
Its general form is as follows:
\begin{align}
    \label{eq:burgers2}
    \pd{u}{t} &= \nu \pd[2]{u}{x} - \frac{1}{2}\pd{}{x}\left( u^2 \right),
\end{align}
where \(\nu \geq 0\) is a constant.
The PDE can be seen as a highly simplified one-dimensional version of a fluid flow problem, with a linear second-order
diffusion term that models the effects of fluid viscosity, and a quadratic convection term that resembles the convection
term in the Navier-Stokes equation.
In this section, Burgers' equation will be used with periodic boundary conditions and a domain length of \(1\),
i.e.~\(u(x + 1, t) = u(x, t)\) for all \(x, t\).

The Burgers equation \eqref{eq:burgers2} is solved with \(\nu = 0.0005\).
The spatial computational domain is discretised using the first-order accurate spatial discretisation given by
Jameson~\cite{jameson2007energy}:
\begin{subequations}
    \begin{align}
        \label{eq:quadratic-flux-discretised}
        \d{\vec{u}_i}{t} = f(\vec{u})_i
        &= \frac{\nu}{\Delta x^2}\left( \vec{u}_{i-1} - 2\vec{u}_i + \vec{u}_{i + 1} \right)
            - \frac{1}{\Delta x}\left( \vec{f}_{i + 1/2} - \vec{f}_{i - 1/2} \right), \\
        \label{eq:quadratic-flux-artificial-diffusion}
        \text{where } \vec{f}_{i + 1/2}
        &= \frac{1}{6}\left( \vec{u}_i^2 + \vec{u}_i\vec{u}_{i+1} + \vec{u}_{i+1}^2 \right)
            - \alpha_{i + 1/2} \left( \vec{u}_{i + 1} - \vec{u}_i \right), \\
        \text{and } \alpha_{i + 1/2}
        &= \frac{1}{4}\abs{\vec{u}_i + \vec{u}_{i + 1}} - \frac{1}{12}\left( \vec{u}_{i + 1} - \vec{u}_i \right).
    \end{align}
\end{subequations}
The resulting ODEs are solved for \(t \in [0, 0.5]\) using the \texttt{Tsit5} algorithm~\cite{tsitouras2011runge}, which
is a fourth-order, five-stage Runge-Kutta method with embedded error estimator.
This ODE solver was chosen following the recommendations of the DifferentialEquations.jl documentation
\cite{recommendedODEsolvers}.
In total, 128 solutions are obtained, each from a random initial state according to \eqref{eq:random-smooth-initial-state}.
Of these solutions, 96 are used for training and the remaining 32 are used for testing.
The solutions are then down-sampled by a factor of 64 in space.
This way, training data is created to allow a neural network to work on the low-fidelity (downsampled) initial
conditions, but to still approximate the original high-fidelity solution.
The downsampling is performed by averaging the solution over chunks of 64 finite volumes.
This is a necessary step, since training data obtained through solving an ODE \(\d{\vec{u}}{t} = f(\vec{u})\) would
trivially allow a neural closure model \(\d{\vec{u}}{t} = f(\vec{u}) + \text{NN}(\vec{u}; \vartheta)\) to produce very
accurate predictions when \(\text{NN}(\vec{u}; \vartheta) \approx 0\).
The coarse-grid solution is saved with a time step of \(\Delta t = 2^{-7}\) between snapshots, meaning that each
coarse-grid solution consists of an initial condition followed by 64 additional snapshots.

Note that when numerically solving PDEs, choosing a first-order accurate spatial discretisation and a fourth-order
accurate ODE solver would typically be a bad choice.
The error in the resulting solution would then be dominated by the error in the spatial discretisation, meaning that
more accurate solutions could be obtained by using a finer spatial discretisation and a lower-order ODE solver.
In the experiments performed in this work the goal is to train neural networks to compensate for the spatial
discretisation error.
This means that choosing a relatively high-order accurate ODE solver is necessary to ensure that the temporal
discretisation error does not contribute significantly to the overall error.

\begin{figure}
    \centering
    \begin{subfigure}[b]{0.48\textwidth}
        \includegraphics[width=\textwidth]{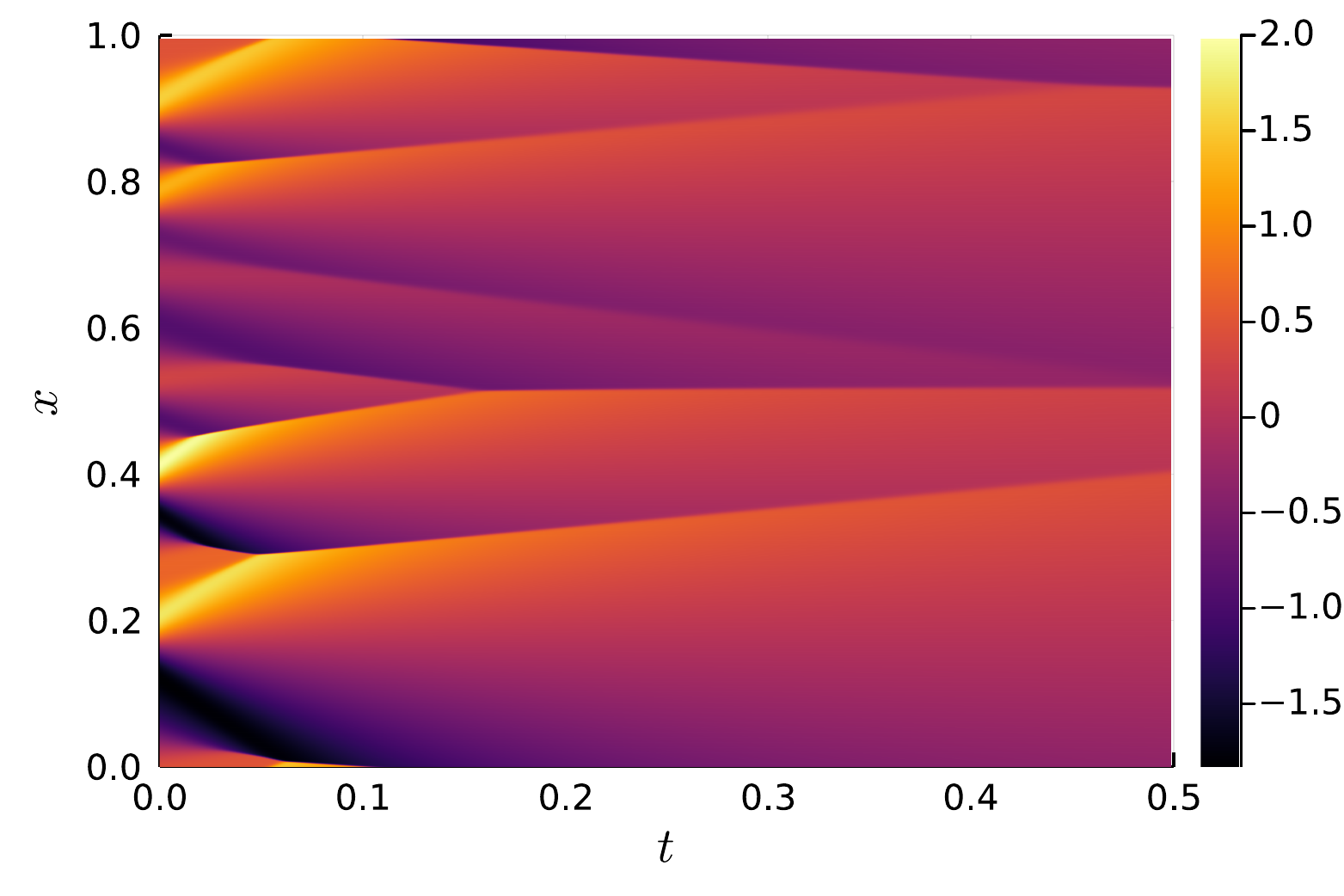}
        \caption{A numerical solution to Burgers' equation on the fine grid (\(N_x = 4096\)).}
        \label{subfig:burgers-data-example-fine}
    \end{subfigure}
    \hfill
    \begin{subfigure}[b]{0.48\textwidth}
        \includegraphics[width=\textwidth]{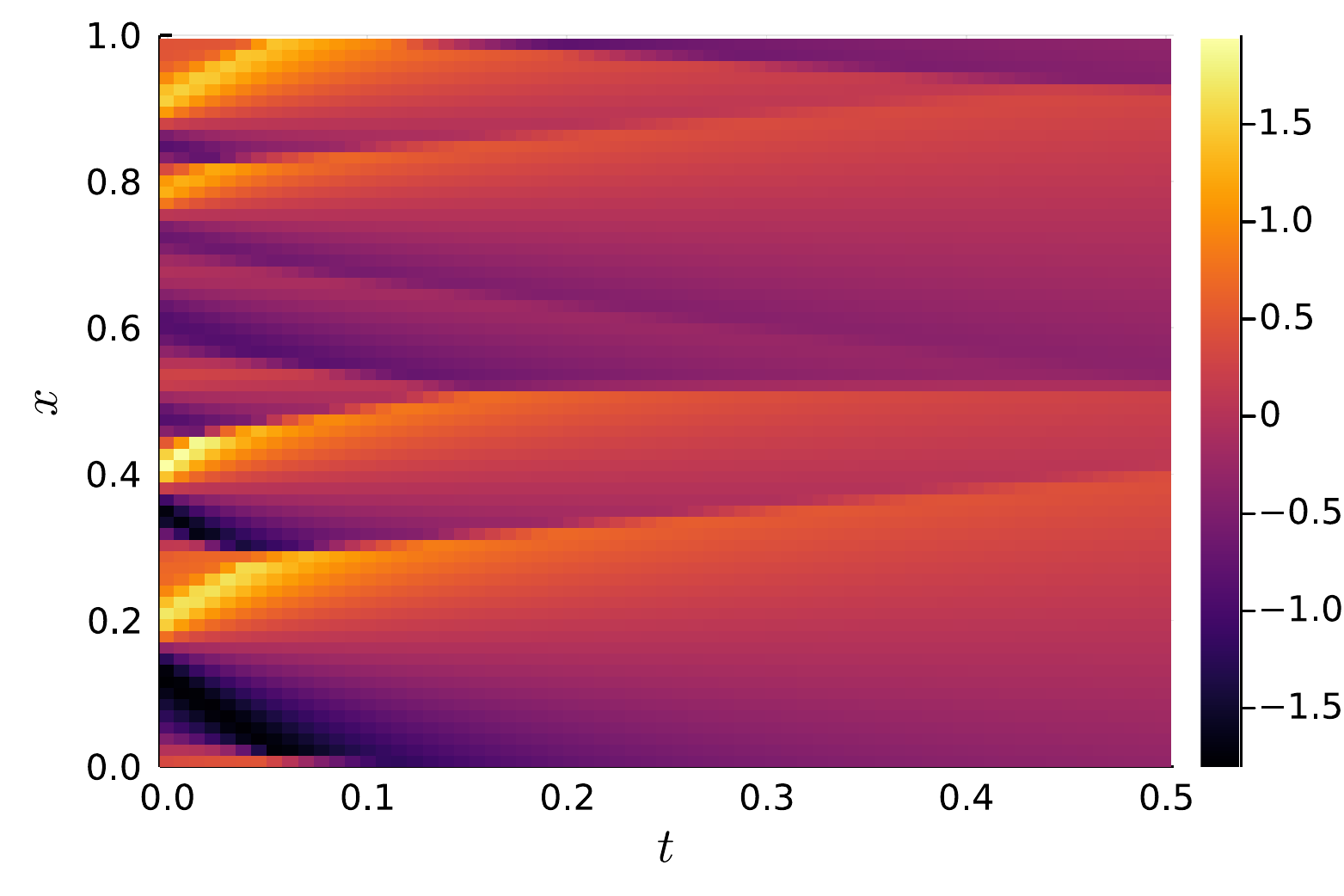}
        \caption{This same solution, downsampled to the coarse grid (\(N_x = 64\)).}
        \label{subfig:burgers-data-example-fine-downsampled}
    \end{subfigure}
    \begin{subfigure}[b]{0.48\textwidth}
        \includegraphics[width=\textwidth]{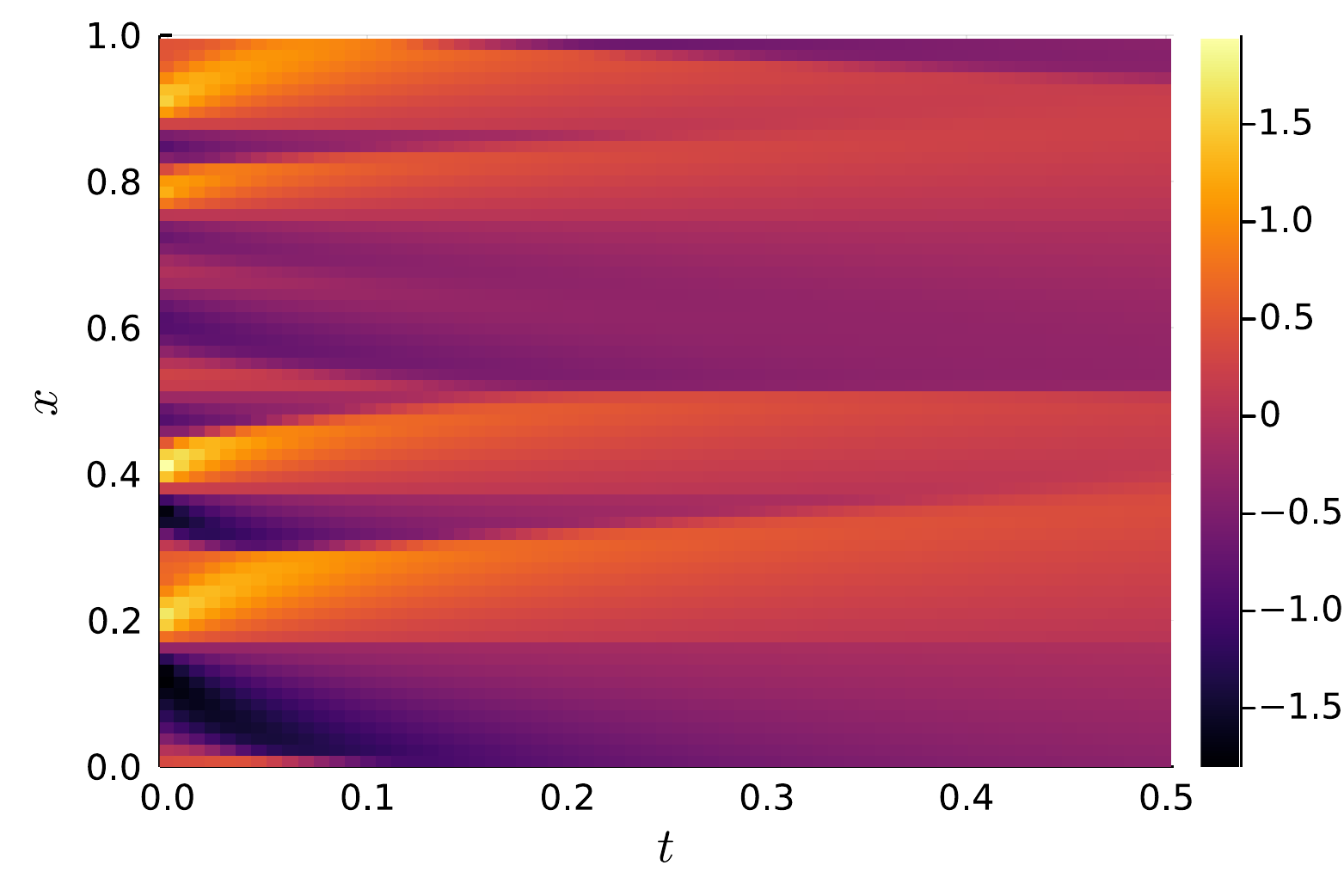}
        \caption{The solution of the ODE on the coarse grid, starting from the downsampled initial condition.}
        \label{subfig:burgers-data-example-coarse}
    \end{subfigure}
    \hfill
    \begin{subfigure}[b]{0.48\textwidth}
        \includegraphics[width=\textwidth]{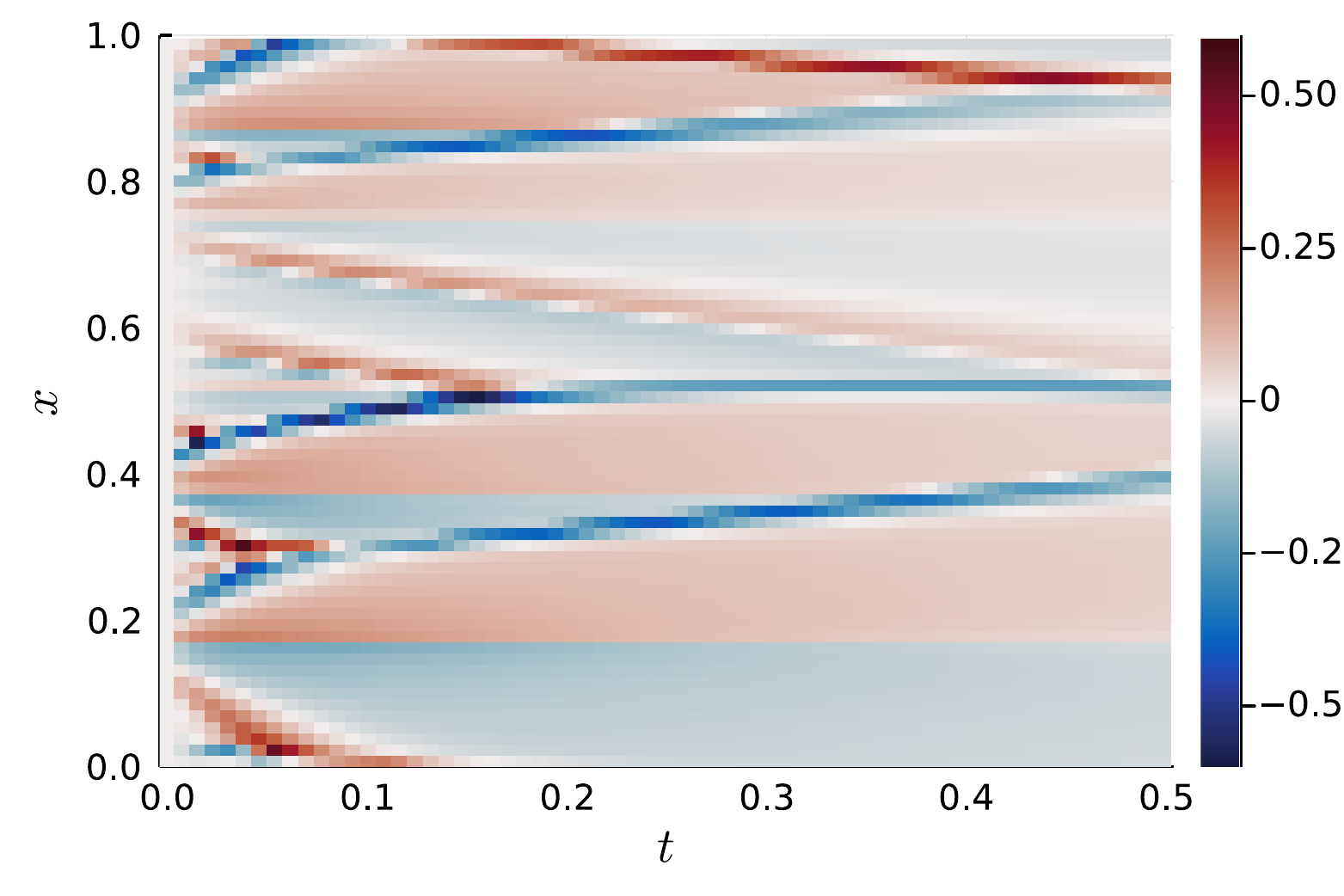}
        \caption{The error between the downsampled fine-grid solution and the coarse grid solution.}
        \label{subfig:burgers-data-example-coarse-error}
    \end{subfigure}
    \caption{
        An example showing how solving a coarse-grid solution to Burgers' equation does not yield the same solution as
        down-sampling the fine-grid solution to the equation.
        A closure model can be added to the coarse discretisation to bring its solution
        (\Cref{subfig:burgers-data-example-coarse}) closer to the fine-grid solution
        (\Cref{subfig:burgers-data-example-fine}).
    }
    \label{fig:burgers-data-example}
\end{figure}

This process is illustrated in \Cref{fig:burgers-data-example}.
This figure also shows the result of solving the coarsely discretised ODE starting from the downsampled initial state
\(\vec{u}(0) \in \R^{64}\) (\Cref{subfig:burgers-data-example-coarse}).
Crucially, this is not equal to the downsampled fine-grid solution
(\Cref{subfig:burgers-data-example-fine-downsampled}).
The downsampled fine-grid solution is by definition the most accurate low-fidelity approximation to the original
high-fidelity data.
The solution of the low-fidelity ODE introduces additional errors (\Cref{subfig:burgers-data-example-coarse-error}),
and is therefore a less accurate approximation to the original data than the downsampled high-fidelity solution.

\subsection{\ks}
\label{sec:data-generation/kuramoto-sivashinsky}
The \ks equation is named after the two researchers who independently derived the equation in
Kuramoto~\cite{kuramoto1978diffusion} and Sivashinsky~\cite{sivashinsky1977nonlinear}.
This PDE is taken with periodic boundary conditions as well:
\begin{subequations}
    \begin{align}
        \label{eq:ks-equation}
        \pd{u}{t} &= -\frac{1}{2}\pd{}{x}\left( u^2 \right) - \pd[2]{u}{x} - \pd[4]{u}{x}, \\
        u(x + L, t) &= u(x, t).
    \end{align}
\end{subequations}

The Lyapunov exponent of \eqref{eq:ks-equation} depends on the length \(L\) of the domain.
Edson \etal\cite{edson2019lyapunov} found the following approximation for the Lyapunov eigenvalues for varying \(L\):
\begin{align*}
    \lambda_i(L) &\approx 0.093 - \frac{0.94}{L}\left( i - 0.39 \right), ~i = 1, 2, \dots \\
    \implies \lambda_{\text{max}}(L) &\approx 0.093 - \frac{0.57}{L}.
\end{align*}
The training data is created with \(L = 64\), leading to a Lyapunov exponent of \(\lambda_{\text{max}} \approx 0.084\).
The inverse of the Lyapunov exponent is the Lyapunov time \(\tlyap\), which can be loosely interpreted
as the time it takes for the error between two similar trajectories to grow by a factor \(e\).
Taking \(L = 64\) yields \(\tlyap \approx 12\).

In this PDE, the time-dependent behaviour of \(u\) is governed by three terms:
\begin{itemize}
    \item
    A non-linear convection term \(-\frac{1}{2}\pd{}{x}\left( u^2 \right)\), the same as in Burgers' equation.

    \item
    A destabilising anti-diffusion term \(-\pd[2]{u}{x}\).
    Note that this term appears on the right-hand with a minus sign, which is unusual for diffusion terms.

    \item
    A stabilising hyper-diffusion term \(-\pd[4]{u}{x}\).
    Without this term, the equation would be ill-posed since the anti-diffusion term would cause solutions to blow up in
    a finite amount of time.
\end{itemize}

Solutions to the \ks equation are created by solving the PDE with domain length \(L = 64\).
As is the case for Burgers' equation, for the \ks equation the function \(u(x, t)\) is discretised as
a time-dependent vector \(\vec{u}(t)\).
Since the non-linear convection term is the same as in Burgers' equation, it is discretised in the same way (see
\eqref{eq:quadratic-flux-discretised}) with the exception that the artificial diffusion term (given by the vector
\(\vec{\alpha}\) in \eqref{eq:quadratic-flux-artificial-diffusion}) is no longer needed since the \ks equation produces
smooth solutions.
The linear diffusion and hyper-diffusion terms are discretised using simple 3-wide and 5-wide stencils, respectively,
leading to the following discretisation for the three different terms:
\begin{align*}
    \left( \frac{1}{2}\pd{}{x}\left( u^2 \right) \right)(\vec{x}_i)
    &\to \frac{1}{\Delta x}\left( \vec{f}_{i + 1/2} - \vec{f}_{i - 1/2} \right), \\
    \text{where } \vec{f}_{i + 1/2}
    &= \frac{1}{6}\left( \vec{u}_i^2 + \vec{u}_i\vec{u}_{i + 1} + \vec{u}_{i + 1}^2 \right), \\
    \pd[2]{u}{x}(\vec{x}_i)
    &\to \frac{1}{\Delta x^2}\left( \vec{u}_{i - 1} - 2\vec{u}_i + \vec{u}_{i + 1} \right), \\
    \pd[4]{u}{x}(\vec{x}_i)
    &\to \frac{1}{\Delta x^4}\left( \vec{u}_{i - 2} - 4\vec{u}_{i - 1}
        + 6\vec{u}_i - 4\vec{u}_{i + 1} + \vec{u}_{i + 2} \right).
\end{align*}
Since the PDE is chosen with periodic boundary conditions, the stencils shown here are also applied with periodic
boundary conditions, i.e.~\(\vec{u}_{i + N_x} = \vec{u}_{i}\).
Written out fully, the resulting ODE is:
\begin{multline}
    \label{eq:ks-equation-discretised}
    \d{\vec{u}_i}{t} = f(\vec{u})_i
    = -\frac{1}{6\Delta x}\left( \vec{u}_{i+1}^2 - \vec{u}_{i-1}^2
        + \vec{u}_i\left( \vec{u}_{i+1} - \vec{u}_{i-1} \right) \right) \\
    \quad - \frac{1}{\Delta x^2}\left( \vec{u}_{i - 1} - 2\vec{u}_i + \vec{u}_{i + 1} \right)
        - \frac{1}{\Delta x^4}\left( \vec{u}_{i - 2} - 4\vec{u}_{i - 1}
        + 6\vec{u}_i - 4\vec{u}_{i + 1} + \vec{u}_{i + 2} \right).
\end{multline}

The PDE is discretised with \(N_x = 1024\) finite volumes, and solved from \(t = 0\) to \(t = 256\).
The initial conditions are generated in the same way as for Burgers' equation, see
\eqref{eq:random-smooth-initial-state}.
The ODEs are solved using the \third-order accurate stiff ODE solver \texttt{Rodas4P}~\cite{steinebach1995order},
again following the recommendations from the DifferentialEquations.jl documentation~\cite{recommendedODEsolvers}.
The resulting solutions are downsampled to \(128\) finite volumes in space with a time step of
\(\Delta t = \frac{1}{2}\) in between snapshots.
To avoid the effects of initial transients, the first 32 snapshots of the trajectories are not used for training.
An example trajectory from the resulting training data is shown in \Cref{fig:example-ks-trajectory}.
\begin{figure}
    \centering
    \includegraphics[width=0.70\textwidth]{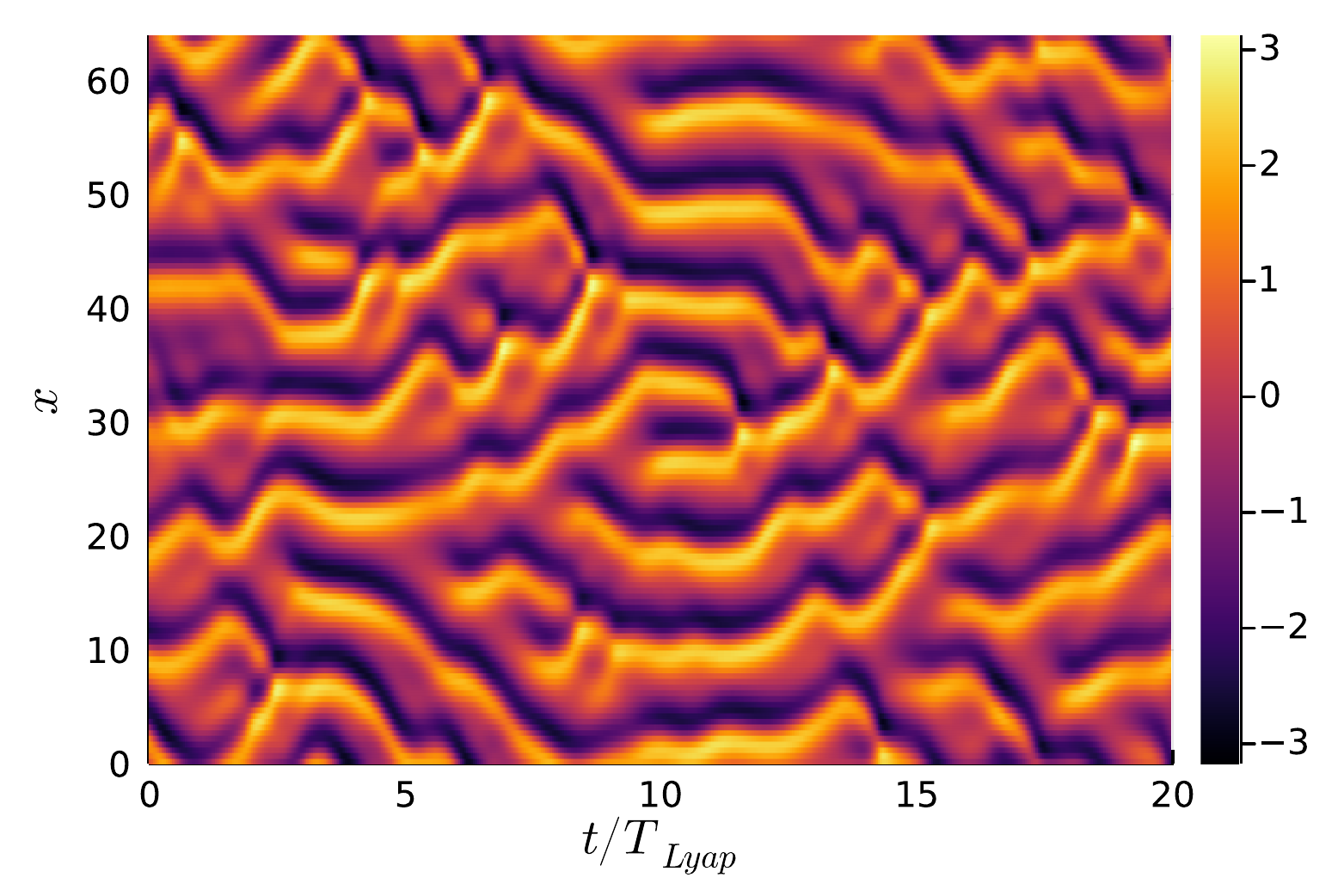}
    \caption{
        An example trajectory from the \ks equation used for training models.
    }
    \label{fig:example-ks-trajectory}
\end{figure}

Of the 100 generated trajectories, the first 80 are used for training and the last 10 are used for testing.
The remaining 10 trajectories are used to evaluate models while they are being trained, to ensure models are trained
for enough iterations but without overfitting.

\section{Training details}
\label{sec:training-details}

\subsection{Neural network architectures}
\label{sec:training-details/nn-architectures}
To ensure that the experiments only test the effect of including prior knowledge, only two different neural networks
are used for all numerical experiments: a `small' convolutional neural network with 57 parameters, and a `large'
convolutional neural network with 533 parameters.

Experiments for the Burgers equation are done only with the smaller of the two neural networks.
Experiments for the \ks equation are done only with the larger neural network.
Note that both neural networks, especially the small neural network, are very small compared to networks used in modern
machine learning applications.
However, as seen from the results, both neural networks are large enough to significantly improve the accuracy over the
coarse ODE without closure term.

The neural network structures are summarised in \Cref{table:small-nn,table:large-nn}.
Note that both neural networks have an initial layer that extends the input vector \(\vec{u}\) with its component-wise
square \(\vec{u}_i^2\).
This is done since the true right-hand sides of both PDEs \eqref{eq:burgers2} and \eqref{eq:ks-equation} contain an
advection term that depends directly on \(u^2\).
As such, passing the values \(\vec{u}_i^2\) to the first convolutional layer is expected to improve the ability of the
network to learn the closure term.
Also note that the single bias parameter in the last convolutional layer of both models is actually meaningless: its
value does not affect the output of the model due to the \(\fwddiff\) layer.
This layer ensures that the entries of the neural network output always sum to zero, which is a property that is also
satisfied by the training data.
Enforcing this property in the neural network was found to result in a small but consistent improvement in accuracy, see
Chapter 4 of Melchers~\cite{mastersthesis}.

\begin{table}
    \centering
    \caption{
        A description of the small neural network structure used for the experiments on Burgers' equation.
    }
    \label{table:small-nn}
    \vspace{1em}
    \begin{tabular}{c c c l}
        \toprule
        Layer & Description & \(\sigma\) & Parameters \\
        \midrule
        1 & \(\vec{u} \mapsto \begin{bmatrix} \left( \vec{u}_i \right)_i & \left( \vec{u}_i^2 \right)_i \end{bmatrix}\) & - & \phantom{2}0 \\
        2 & \(9\)-wide conv, \(2 \to 2\) channels & \(\tanh\) & 38 (\(9 \times 2 \times 2\) weights and \(2\) biases) \\
        3 & \(9\)-wide conv, \(2 \to 1\) channels & -         & 19 (\(9 \times 2 \times 1\) weights and \(1\) bias) \\
        8 & \(\vec{u} \mapsto \fwddiff\vec{u}\)   & \(-\)     & \phantom{1}0 \\
        \midrule
        Total & & & 57 \\
        \bottomrule
    \end{tabular}
    \vspace{3em}
    \caption{
        A description of the large neural network structure used for the experiments on the \ks equation.
    }
    \label{table:large-nn}
    \vspace{1em}
    \begin{tabular}{c c c l}
        \toprule
        Layer & Description & \(\sigma\) & Parameters \\
        \midrule
        1 & \(\vec{u} \mapsto \begin{bmatrix} \left( \vec{u}_i \right)_i & \left( \vec{u}_i^2 \right)_i \end{bmatrix}\) & - & \phantom{12}0 \\
        2 & \(5\)-wide conv, \(2 \to 4\) channels & \(\tanh\) & \phantom{1}44 (\(5 \times 2 \times 4\) weights and \(4\) biases) \\
        3 & \(5\)-wide conv, \(4 \to 6\) channels & \(\tanh\) &           126 (\(5 \times 4 \times 6\) weights and \(6\) biases) \\
        4 & \(5\)-wide conv, \(6 \to 6\) channels & \(\tanh\) &           186 (\(5 \times 6 \times 6\) weights and \(6\) biases) \\
        5 & \(5\)-wide conv, \(6 \to 4\) channels & \(\tanh\) &           124 (\(5 \times 6 \times 4\) weights and \(4\) biases) \\
        6 & \(5\)-wide conv, \(4 \to 2\) channels & \(\tanh\) & \phantom{1}42 (\(5 \times 4 \times 2\) weights and \(2\) biases) \\
        7 & \(5\)-wide conv, \(2 \to 1\) channels & \(-\)     & \phantom{1}11 (\(5 \times 2 \times 1\) weights and \(1\) bias) \\
        8 & \(\vec{u} \mapsto \fwddiff\vec{u}\)   & \(-\)     & \phantom{12}0 \\
        \midrule
        Total & & & 533 \\
        \bottomrule
    \end{tabular}
\end{table}
All neural networks are trained using the Adam optimiser~\cite{kingma2014adam} with a learning rate of \(10^{-3}\).

\subsection{\Discopt for the \ks equation}
\label{sec:ks-pseudospectral}
As mentioned in \Cref{sec:gradient-forms/discopt}, the \discopt approach requires the use of a differentiable
ODE solver.
This is not a problem for Burgers' equation, but does pose a problem for the stiff \ks equation since
stiff equations are typically solved using implicit methods, which are not trivial to back-propagate through.
Note that back-propagating through implicit methods is possible, since the gradient of an implicitly defined function
(i.e.~a function whose output is defined as the solution of a system of equations) can be computed by the implicit
function theorem, as demonstrated by Kolter \etal\cite{implicitlayers}.
Nevertheless, explicit ODE solvers are preferable over implicit methods whenever they are applicable, due to their
simplicity and speed, as well as the property that back-propagation through explicit methods is comparatively easy.

For some problems including the \ks equation, exponential time differencing Runge-Kutta methods are suitable.
These methods assume a stiff but linear term which can be solved exactly using exponentials, combined with a
non-stiff non-linear term that can be taken into account using multiple stages, similar to how standard explicit
Runge-Kutta methods achieve higher orders of accuracy.
Exponential integrators of orders 2, 3, and 4 were derived by Cox and Matthews~\cite{cox2002exponential}.
A numerically stable way to compute the coefficients required by these methods was presented by Kassam and
Trefethen~\cite{kassam2005fourth} and demonstrated on the four-stage fourth-order accurate method \texttt{ETDRK4}.
The resulting algorithm was found to perform very well on a variety of problems including the \ks equation and will
therefore be used here.

Exponential integrators for the \ks equation are most efficient when the PDE is solved in the pseudo-spectral domain,
meaning that the ODE is not over the variables \(\vec{u}(t)\), but over their discrete Fourier transform \(\hat{\vec{u}}
(t) \bydef \mathcal{F}\vec{u}(t)\) (the Fourier transform is only performed over space, not over time).
Transforming the equation in this way yields the following ODE system:
\begin{align}
    \label{eq:ks-pseudo-spectral-no-closure}
    \d{}{t}\hat{\vec{u}} &=
    \left( \mat{\Lambda}^2 - \mat{\Lambda}^4 \right)\hat{\vec{u}}
        - \frac{i}{2}\mat{\Lambda}\mathcal{F}\left( \left( \mathcal{F}^{-1}\hat{\vec{u}} \right)^2 \right),
\end{align}
where \(\mat{\Lambda}\) is a diagonal matrix \(\mat{\Lambda} = \text{diag}\left( \lambda_0, \lambda_1, \dots,
\lambda_{N_x - 1} \right)\) where
\begin{align*}
    \lambda_k &= \begin{cases}
        \frac{2\pi k}{L} &\text{for } 0 \leq k < \frac{N_x}{2}, \\
        0 &\text{for } k = \frac{N_x}{2}, \\
        \frac{2\pi}{L}(k - N_x) &\text{for } \frac{N_x}{2} < k \leq N_x - 1.
    \end{cases}
\end{align*}

\end{document}